\documentclass[11pt]{article}
\usepackage[margin=1in]{geometry}
\usepackage{amsfonts}
\usepackage{amsmath,amsthm,amssymb}
\usepackage{latexsym}
\usepackage{epic}
\usepackage{epsfig}
\usepackage[colorlinks=true, allcolors=magenta]{hyperref}
\usepackage{verbatim}
\usepackage[justification=centering]{caption}
\usepackage{enumitem}
\usepackage[dvipsnames]{xcolor}
\allowdisplaybreaks[1]
\usepackage[style=alphabetic,natbib=true,maxbibnames=99,minalphanames=4,maxalphanames=4]{biblatex}
\usepackage{algorithm}
\usepackage[most]{tcolorbox}
\tcbset{rounded corners}
\usepackage{setspace}
\usepackage[noend]{algpseudocode}
\algrenewcommand\algorithmicrequire{\textbf{Input:}}
\algrenewcommand\algorithmicensure{\textbf{Output:}}
\usepackage{thmtools} 
\usepackage{thm-restate}

\usepackage[normalem]{ulem}

\usepackage{ifthen}
\newboolean{submission}
\setboolean{submission}{false}

\newtheorem{theorem}{Theorem}[section]

\newtheorem{lemma}[theorem]{Lemma}

\newtheorem{definition}[theorem]{Definition}

\newtheorem{question}{Question}

\newenvironment{alg}{\begin{algorithm}\begin{onehalfspace}\begin{algorithmic}[1]}{\end{algorithmic}\end{onehalfspace}\end{algorithm}}

\renewcommand{\hat}{\widehat}

\def\min{\qopname\relax n{min}}
\def\max{\qopname\relax n{max}}
\def\argmin{\qopname\relax n{argmin}}

\newcommand{\expect}[1]{%
  \mathop{\qopname\relax n{\mathbb{E}}}\!\left[#1\right]
}
\newcommand{\expected}[1]{%
  \qopname\relax n{\mathbb{E}}_{#1}
}
\newcommand{\expects}[2]{%
  \qopname\relax n{\mathbb{E}}_{#1}\!\left[#2\right]
}

\newcommand{\prob}[1]{%
  \mathop{\qopname\relax n{\mathbb{P}}}\!\left[#1\right]
}

\newcommand{\probed}[1]{%
  \qopname\relax n{\mathbb{P}}_{#1}
}
\newcommand{\probs}[2]{%
  \qopname\relax n{\mathbb{P}}_{#1}\!\left[#2\right]
}

\newcommand{\NN}{\mathbb{N}}

\newcommand{\HH}{\mathbb{H}}

\def\A{\mathcal{A}}
\def\B{\mathcal{B}}
\def\C{\mathcal{C}}
\def\D{\mathcal{D}}

\def\F{\mathcal{F}}

\def\H{\mathcal{H}}

\def\K{\mathcal{K}}

\def\O{\mathcal{O}}

\def\Q{\mathcal{Q}}

\def\X{\mathcal{X}}

\def\eps{\epsilon}
\def\sse{\subseteq}

\newcommand{\eat}[1]{}

\newenvironment{lp*}{\begin{equation*}  \begin{array}{lll}}{\end{array}\end{equation*}}

\usepackage{color-edits}
\addauthor[Alireza]{az}{OliveGreen}
\addauthor[Shaddin]{sd}{red}

\newcommand{\ig}{\mathsf{IG}}

\newcommand{\simnr}{\sim_\mathsf{nr}}

\newcommand{\dom}{\mathsf{dom}}

\newcommand{\dtv}{\mathsf{d_{TV}}}
\newcommand{\supp}{\mathsf{supp}}

\newcommand{\tunif}{\mathsf{MTestUnif}}
\newcommand{\tunifstandard}{\mathsf{TestUnif}}
\newcommand{\mtest}{\mathsf{Test}}
\newcommand{\packnum}{\mathsf{pack}}
\newcommand{\sd}{\mathsf{d}}
\newcommand{\cons}{\mathsf{cons}}
\newcommand{\sss}{\mathsf{SS}}

\newcommand{\diam}{\mathsf{diam}}

\DeclareMathAlphabet{\bbold}{U}{bbold}{m}{n}
\newcommand{\id}{\ensuremath{\bbold{1}}}
\newcommand{\hone}{\ensuremath{h_\mathbf{1}}}
\newcommand{\hzero}{\ensuremath{h_\mathbf{0}}}

\newcounter{alg}
\renewcommand{\thealg}{\arabic{alg}}
\newtcolorbox[auto counter]{algoboxinner}[1][]{
  colback=white,
  colframe=black,
  rounded corners,
  boxrule=0.8pt,
  arc=8pt,
  left=6pt,right=6pt,top=6pt,bottom=6pt,
  before upper={
    \refstepcounter{alg}%
    \textbf{Algorithm~\thealg.}~\textbf{#1}
  }
}

\newenvironment{algobox}[1][tpb]{%
  \begin{figure}[#1]
  \begin{algoboxinner}
}{%
  \end{algoboxinner}
  \end{figure}
}

\title{Relatively Smart II: \\ Tractable or Semi-Supervised Instance-Optimal Learning}

\author{
 Shaddin Dughmi\thanks{shaddin@usc.edu. Supported by NSF Grant CCF-2432219. Part of this work was done while the author was on sabbatical as the Carter and Tania Neild visiting professor at Northwestern University, as well as a visiting professor in the Data Science Institute at the University of Chicago.}\\
 University of Southern California
 \and 
 Alireza F. Pour\thanks{alireza.fathollahpour@uwaterloo.ca. Supported by a David Cheriton Scholarship and a Vector Institute Research Grant. }\\
 University of Waterloo}

\date{}

\begin{document}

\maketitle
\begin{abstract}
  We continue the study of relatively smart learning as introduced by \citet{pmlr-v336-dughmi26a}, which asks a
supervised learner to compete, marginal by marginal, with every distribution-fixed error guarantee that can
be soundly certified from unlabeled data. They showed that the One-Inclusion Graph (OIG) learner is relatively smart with a quadratic blowup in sample complexity, and no relatively-smart learner can achieve better sample complexity. However, it was left open whether Empirical Risk
Minimization (ERM), or another natural or tractable learner, enjoys comparable guarantees. Moreover, whether the blowup in sample complexity can be restricted to unlabeled data---leading to a relatively smart semi-supervised learner---was also left as a tantalizing possibility.

Our first result resolves the first question positively:  We show that ERM---and in fact any proper consistent learner---is relatively smart for binary classification in the distribution-free setting. Our proof shows that a small certifiable error when allowed $m$ samples implies a similarly small error on the uniform distribution over a random sample of size $O(m^2)$. This yields the existence of a cover of size at most $2^{m+1}$ on that sample, which suffices to control the error of every consistent learner when given $O(m^2)$ samples.

We then explore the semi-supervised question: whether adding only unlabeled data suffices for relatively smart learning, and  how this affects the tractability/simplicity of learning. We show a positive result for the information-theoretic question: there is a semi-supervised relatively smart learner provided a quadratic blowup only in unlabeled sample complexity, whereas the labeled sample complexity suffers no blowup.  The learner is based on a natural generalization of the one-inclusion-graph to a leave-most-out transductive problem, in which the labels of a subset of a finite pool are revealed and the learner predicts the labels of the remaining points.

Finally, we show that---in contrast to our first result on ERM---this information-theoretic label efficiency comes at the cost of simplicity and tractability of the learner: When a learner can only access the hypothesis class through calls to an agnostic ERM oracle, we prove that any semi-supervised relatively smart learner with substantially sub-quadratic blowup in labeled sample complexity must make super-polynomially many oracle calls. This holds even when the marginal distribution is given explicitly, and therefore implies a similar intractability result for distribution-fixed learning which may be of independent interest.
\end{abstract}

\section{Introduction}
\label{sec:intro}
The classical PAC model seeks \emph{fully supervised} learning guarantees that hold uniformly over both the marginal distribution of unlabeled data as well as the labeling hypothesis, and consequently measures performance in the worst case over all admissible instances. However, the difficulty of learning can vary substantially from one marginal distribution to another, even when the hypothesis is still chosen adversarially. The \emph{distribution-fixed} model of learning, studied by \citet{benedek_learnability_1991}, captures this by giving the learner complete knowledge of the marginal distribution~$\D$, allowing its learning strategy to be tailored specifically to $\D$. Distribution-fixed learning can therefore be viewed as an idealized form of semi-supervised learning, in which the learner has access not merely to additional unlabeled samples but to the marginal itself \citep{ben2008does}. 

One may then ask whether the advantages of distribution-fixed learning can be recovered without giving the learner full knowledge of the marginal distribution. \emph{Smart PAC learning}, introduced by \citet{darnstadt2011smart_journal}, seeks a single supervised learner whose error rate on each marginal approaches that of the optimal distribution-fixed learner for that marginal, despite not being given the marginal distribution itself. In this sense, smart learning seeks an instance-optimal guarantee with respect to the marginal distribution. \citet{darnstadt2011smart_journal} show that such a learner exists for ``most'' marginals, where ``most'' is quantified with respect to an arbitrary and known fixed prior distribution over marginals. However, this qualification cannot in general be removed: \citet{darnstadt_unlabeled_2013}, strengthening a result of \citet{dudley_metric_1994}, show that there are hypothesis classes and families of marginals for which the optimal distribution-fixed error rates converge rapidly and uniformly to zero, yet for every supervised learner that is not given the marginal and every finite sample size, there exists a marginal on which its error remains essentially that of random guessing.

In a recent work, \citet{pmlr-v336-dughmi26a} identify the obstruction behind this impossibility as an indistinguishability phenomenon: a learner $\A_{\D}$ tailored to a marginal $\D$ may perform much worse on another marginal $\D'$ that is indistinguishable from $\D$ using unlabeled data. In such settings, the learner cannot reliably identify whether the observed data comes from a marginal on which its distribution-specific learning strategy remains valid, or from another statistically similar marginal on which that strategy fails. Thus, unlabeled data may be insufficient to determine whether the favorable error guarantee associated with $\D$ can actually be trusted.

Inspired by this, \citet{pmlr-v336-dughmi26a} introduce \emph{relatively smart learning}. Rather than competing with arbitrary distribution-fixed error rates, they require a learner to compete only with error guarantees that can be \emph{soundly certified from unlabeled data}. Informally, a learner $\A_{\D}$ tailored to $\D$ is paired with a certifier $\C$ that observes an unlabeled sample and estimates $\A_{\D}$'s error. Although $\A_{\D}$ may be designed specifically for $\D$, soundness requires the expected value of the certifier to upper bound the error of $\A_{\D}$ under \emph{every} admissible marginal $\D'$. Thus, if $\D$ is difficult to distinguish from a marginal on which $\A_{\D}$ has substantially larger error, the certifier is forced to account for this ambiguity. The resulting \emph{certifiable error rates} capture precisely the distribution-dependent guarantees that can be justified from unlabeled data, and relatively smart learning asks for a single learner that competes with all such guarantees, marginal by marginal.

In the distribution-free setting, \citet{pmlr-v336-dughmi26a} establish an essentially tight information-theoretic characterization. They show that the One-Inclusion Graph (OIG) learner is relatively smart, up to a constant blow-up in error and a quadratic blow-up in sample complexity. They further show that every relatively smart learner must suffer a near-quadratic blow-up in the sample complexity.

The positive result of \citet{pmlr-v336-dughmi26a} relies on the One-Inclusion Graph (OIG) learner and, more specifically, on its dataset-by-dataset optimality for the leave-one-out transductive problem. Standard learning rules such as Empirical Risk Minimization (ERM) do not share this property: there are datasets on which ERM and OIG have very different leave-one-out behavior, despite their closely related guarantees in PAC learning. It is therefore natural to ask whether the transductive optimality underlying the OIG analysis is essential for relatively smart learning with quadratic (or even bounded) blow-up, or whether a simpler and more standard learning rule can obtain a comparable guarantee. 
\begin{question}[Open Question 3.3 of \citet{pmlr-v336-dughmi26a}] Is ERM, or another natural and typically tractable learner, relatively smart?
\end{question}

A second question concerns the role of unlabeled data in relatively smart learning. Since distribution-fixed learning can be viewed as an idealized form of semi-supervised learning, it is natural to ask whether access to additional unlabeled samples can help a learner compete with certifiable error rates. The lower bound of \citet{pmlr-v336-dughmi26a} applies to fully supervised learners and therefore does not determine how much of the required sample-complexity blow-up must remain in labeled data once unlabeled samples are available. This leaves open the possibility that unlabeled data may substantially reduce the number of labels needed, perhaps even eliminating the blow-up in labeled sample complexity altogether.

\begin{question}
Can unlabeled data substantially reduce the labeled sample complexity of relatively smart learning? More specifically, are there semi-supervised relatively smart learners with no blow-up in labeled sample complexity?
\end{question}

Our first question asks whether relatively smart learning can be achieved by a simple and typically-tractable supervised learning rule such as ERM. The semi-supervised question above raises a different possibility: using unlabeled data to reduce the labeled sample complexity may require a substantially more complicated learning procedure, even if the corresponding information-theoretic guarantee is achievable. Thus, beyond asking whether unlabeled data can reduce the number of labels, it is natural to ask whether this improvement can be achieved by a tractable learner.

\begin{question}
Can a semi-supervised relatively smart learner with substantially improved labeled sample complexity be implemented tractably?
\end{question}

A positive answer to the second question coupled with a negative answer to the third  would reveal a separation between the information-theoretic and algorithmic value of unlabeled data. It would mean that, although unlabeled data may be sufficient to eliminate the blow-up in labeled sample complexity, exploiting this advantage can require a substantially more complex learning procedure. In particular, the obstacle would no longer be obtaining information about the marginal, but rather converting that information into an effective learning strategy.

\subsection{Results and Techniques}
We answer the first two questions in the affirmative and give a strong negative answer to the third in a natural oracle model.

\paragraph{ERM is relatively smart.}
Our first result resolves the open question of \citet{pmlr-v336-dughmi26a} concerning ERM. We show that, in the distribution-free setting, ERM is relatively smart with a quadratic blow-up in sample complexity. In fact, the same guarantee holds for every proper consistent learner: with quadratically more labeled samples, any such learner can approximately compete, up to a constant-factor loss, with the best certifiable error guarantee at the original sample size.

The reason is fundamentally different from the one underlying the OIG result. In \citet{pmlr-v336-dughmi26a}, sound certification implies a lower bound on the certifiable error rate in terms of the optimal learning error on uniform distributions over random finite samples. The dataset-by-dataset leave-one-out optimality of OIG then provides a direct way to compete with this lower bound. ERM has no comparable transductive optimality; indeed, there are finite samples on which its leave-one-out behavior is substantially worse than that of OIG.\footnote{For example, consider the class of behaviors with hamming weight at most one on a sufficiently large dataset. An ERM learner that breaks ties against the all-zero behavior has much worse leave-one-out error than OIG.} Our result therefore shows that leave-one-out optimality is sufficient for relatively smart learning, but not necessary.

Our proof instead derives a different consequence of a small certifiable error rate. As in the original argument, for a marginal $\D$ and sample size $m$, we first show that on average over a random sample $S$ of size $O(m^2)$, the error of the $\D$-fixed learner $\A_{\D}$ on the uniform distribution over $S$, denoted by $\D_S$, is not much larger than the certifiable error rate. We then show that if $\A_{\D}$ achieves small error on $\D_S$, then the hypothesis class cannot contain too many hypotheses that are pairwise well separated under $\D_S$. Equivalently, the relevant packing number on $\D_S$ is small, which yields a small cover of the class. Finally, a double-sampling argument shows that to bound the error of a consistent learner on $\D$, it is sufficient to prove that any hypothesis consistent with a random half of $S$ has small error on the other half. The cover of $\D_S$ is small enough for any hypothesis in the cover to enjoy similar error on both halves, and we exploit the coverage to show that this extends to every hypothesis in class.

\paragraph{Semi-supervised relatively smart learning with no labeled sample-complexity blow-up.}
Our second result answers the semi-supervised question in the strongest possible form. We introduce the notion of relatively smart semi-supervised learning that separately tracks unlabeled and labeled sample complexity, and show that the entire sample-complexity blow-up can be shifted to unlabeled data. In particular, we show a semi-supervised learner that can approximately compete with the best certifiable error guarantee at a given sample size while using only the same number of labeled samples, together with quadratically more unlabeled samples. Thus, unlike the fully supervised setting, relatively smart semi-supervised learning requires no blow-up at all in labeled sample complexity.

The construction again builds on the connection between certifiable error rates and learning under a uniform distribution over a large random sample.  
Given the unlabeled sample, the labeled sample, and a test point, we form their union $T$ and consider the uniform distribution $\D_T$ over $T$. The learner then uses an optimal $\D_T$-fixed learner $\A^*_T$, trained only on the available labeled sample, to predict the test point. Let $\A_{\D}$ be the $\D$-fixed learner witnessing the certifiable error rate on $\D$. We show that, on average over $T$, the error of $\A_{\D}$ on $\D_T$ when given $m$ labeled samples serves as a lower bound on the certifiable error rate. The optimality of $\A^*_T$ for $\D_T$ then shows that our learner's error is competitive with the certifiable error rate.

The semi-supervised construction also admits a natural transductive interpretation. The learner above induces a problem in which only a small subset of a finite pool is labeled and the learner must predict the labels of the remaining points. We refer to this as a \emph{leave-most-out} problem. We define a natural generalization of the OIG for this setting and show that an appropriate ``orientation'' achieves the optimal leave-most-out error. Hence, the information-theoretically optimal semi-supervised learner can be realized through a corresponding Inclusion Graph construction, paralleling the role of OIG in the leave-one-out setting.

\paragraph{A tractability barrier for relatively smart learning with no blow-up.}

The semi-supervised learner above is highly intractable: it requires computing the optimal distribution-fixed learner, or equivalently an optimal leave-most-out orientation, for the observed sample, and there is no general recipe for carrying this out efficiently. To study whether the same improvement in labeled sample complexity can be achieved by a tractable algorithm, we formalize access to the hypothesis class through an \emph{agnostic ERM oracle}. Given any finite labeled sample, such an oracle returns a hypothesis in the class minimizing empirical error. We otherwise place no restriction on the learner's computation, and call it efficient if it makes only polynomially many oracle calls.

We show that no efficient oracle learner can be relatively smart with any substantially subquadratic blow-up in labeled sample complexity, regardless of the amount of unlabeled data available to it. The lower bound continues to hold even when the marginal distribution itself is given to the learner. As a consequence, our construction also yields a tractability barrier for distribution-fixed learning itself: even with complete knowledge of the marginal, polynomially many agnostic ERM queries do not suffice to approach the optimal distribution-fixed error with substantially subquadratic sample overhead. This stands in contrast to PAC learning, where optimal sample complexity can be attained by learners constructed from aggregations of ERM solutions \citep{hanneke_optimal_2016,larsen_bagging_2023,aden-ali_optimal_2023,aden-ali_majority--three_2024,rawal2026majority}.

At a high level, our construction starts from the lopsided classes used in the lower bounds of \citet{pmlr-v336-dughmi26a} defined as follows: on a domain of size $n$ consider all hypotheses that assign a minority label to at most $o(n)$ points. Under the uniform distribution on the domain, every hypothesis in such a class is nearly constant, with only a vanishing fraction of points receiving the minority label. We augment this class with an additional random hypothesis $f \in \{0,1\}^n$. If $f$ is explicitly known, a learner can simply choose between $f$ and the two constant predictors and thereby achieve vanishing error of $o(n)/n$ on nearly uniform distributions; moreover, this performance can be soundly certified for sample sizes $m\approx \sqrt{n}$ by exploiting techniques from uniformity testing. An oracle learner, however, must discover $f$ through its ERM queries. We construct a valid oracle that breaks ties in favor of the lopsided class and show that, for any fixed query, only an exponentially small fraction of choices of $f\in\{0,1\}^n$ can force the oracle to reveal it. To extend this to adaptively chosen queries, suppose $f$ is ever returned and consider the first round $t$ at which this happens. By the oracle's tie-breaking rule, in every preceding round the oracle returns a member of the lopsided class according to a rule that is independent of the unseen labels of $f$. Hence, after fixing the observed labels and the learner's internal randomness, the interaction before round $t$, and therefore the query made at round $t$, is fixed with respect to the remaining randomness of $f$. The fixed-query bound can therefore be applied for round $t$, and a union bound over polynomially many rounds shows that the oracle reveals $f$ with only exponentially small probability.
Consequently, with high probability over the random choice of $f$, the oracle never reveals $f$, so the learner gains no information about its unseen labels beyond those already contained in the labeled sample. Its prediction on an unseen test point is therefore essentially independent of the corresponding label of $f$, forcing error close to $1/2$. Since this occurs with high probability over the random choice of $f$, there must exist a fixed hypothesis $f^*$ on which the learner fails to approach the vanishing certifiable error rate.

\subsection{Related Work}
{\bf Distribution-fixed and semi-supervised learning.} Distribution-fixed learning, where the learner is given the full knowledge of marginal,  was studied by \citet{benedek_learnability_1991}. They show that the optimal distribution-fixed learning rate is intimately related to finite covers of the hypothesis class, at every error scale, with respect to the disagreement metric induced by the marginal. Distribution-fixed learning can be viewed as an idealized form of semi-supervised learning, in which the learner has complete access to the marginal rather than only finitely many unlabeled samples \citep{ben2008does}. A large literature studies more realistic formulations of semi-supervised settings and asks when finite unlabeled data can improve learning, and by how much \citep[e.g.][]{ben2008does,balcan2010discriminative,darnstadt_unlabeled_2013,globerson_effective_2017,gopfert_when_2019,pmlr-v97-golovnev19a,pukdee2023learning}. At a high level, unlabeled data cannot improve worst-case minimax rates in the fully distribution-free PAC setting, but it can yield substantial improvements on particular marginals, as well as under restrictions on the family of marginals or suitable compatibility assumptions between the marginal and the hypothesis class.

\paragraph{Smart and relatively smart learning.}
Smart PAC learning, introduced by \citet{darnstadt2011smart_journal}, asks for a supervised learner that competes, marginal by marginal, with the optimal distribution-fixed learner. While they show that smart learning is possible for most marginals with respect to a known prior, it is impossible more generally \citep{dudley_metric_1994,darnstadt_unlabeled_2013}. \citet{pmlr-v336-dughmi26a} attribute this impossibility to an indistinguishability phenomenon and introduce relatively smart learning, which restricts the benchmark to distribution-dependent guarantees that can be soundly certified from unlabeled data. \citet{pmlr-v336-dughmi26a} show that OIG is relatively smart with a quadratic sample-complexity blow-up and that no substantially subquadratic blow-up is possible in general. We continue this line by showing that ERM, and more generally every proper consistent learner, is relatively smart. We then introduce a semi-supervised formulation of relative smartness and show that unlabeled data can eliminate the blow-up in labeled sample complexity entirely. Finally, we prove that this label efficiency cannot be achieved by tractable general-purpose learners; that is, learners which can only access the hypothesis class by making  polynomially many agnostic ERM-oracle calls.

\paragraph{Testable learning.}
Relatively smart learning is also closely related to the framework of testable learning introduced by \citet{rubinfeld2023testing} and developed further in, e.g., \citep{gollakota2023moment,gollakota2023tester,  diakonikolas2023efficient,klivans2024testable,gollakota2024efficient}. The certifiers in relatively smart learning are real-valued analogues to the testers in testable learning and the soundness requirement is similar to the requirement in testable learning that the marginals that satisfy the assumption must pass the test. Testable learning is concerned with the design of learner/tester pairs for a given distributional property, whereas in relatively smart learning those are given as the benchmarks for every distribution separately, and a relatively smart learner must compete with every distribution's benchmark without knowing that distribution in advance.

There are also several differences in emphasis. Testable learning has primarily focused on computational efficiency, with \citet{gollakota2023moment} as a notable information-theoretic exception. Testers in this literature are permitted to use labeled data, although many existing results do not require this. Finally, the framework has largely been studied in the agnostic setting; in the realizable setting, access to labeled data makes the analogous testing problem trivial, except in settings such as distribution shift \citep{klivans2024testable}.

\section{Preliminaries} \label{sec:prelims}
{\bf Notation.} For a domain $\X$, we use $\Delta(\X)$ to denote the family of probability distributions supported on $\X$. For a distribution $\D \in \Delta(\X)$ and an event $Y \sse \X$, we use $\D[Y]=\probs{\D}{Y}$ 
as shorthand for the probability of~$Y$, and $\D_{|Y} = \probs{\D}{\ .\ | Y}$ as shorthand for the conditional distribution of $\D$ given~$Y$. When the domain $\X$ is countable we use $\D[x]=\probs{\D}{x}$ to denote the probability of $x\in \X$, and use $\supp(\D) = \{x \in \X: \D[x] > 0\}$ to denote the support of $\D$. For a finite multiset $S$ we use $\D_S$ to denote the uniform distribution on~$S$. Moreover, for any finite multiset $T\in (\X\times\{0,1\})^*$, we denote by $\dom(T)\in\X^*$ the unlabeled part of $T$.  Given a function $h$ defined on some domain $\X$, we use $h|Y$ to denote its restriction to some $Y \sse \X$. Moreover, for a hypothesis class $\H\subseteq \{0,1\}^{\X}$ we denote $\H|Y = \{h|Y: h\in \H\}$. For a set $\X$ we use $\X^*$ to denote the family of finite sequences over $\X$. For a predicate or probability event $E$ we use $\id[E] \in \{0,1\}$ to denote the indicator of $E$. Finally, we use standard order-of-growth (big-Oh) notation, though for functions $f(\eta,m)$ and $g(m)$  we write $f= O_\eta(g)$ to indicate that $f=O(g)$ whenever $\eta$ is a fixed constant. \looseness=-1

We work in the setting of binary PAC learning. There is a data domain $\X$ and a hypothesis class
$\H\sse\{0,1\}^{\X}$.  The unlabeled data comes from a marginal distribution $\D$ on $\X$ and is labeled by a hypothesis from $\H$. A marginal distribution $\D$ on $\X$ together with a target $h\in\H$ induces the
joint distribution $\D_h$ on labeled data $(x,y)$ with $x \sim \D$ and $y = h(x)$. In this work, we focus on the \emph{distribution-free setting} where $\D$ can be arbitrary, whereas relatively smart learning can be defined for a \emph{distribution-class} as well; see \citet{pmlr-v336-dughmi26a}. For a predictor $f:\X\to\{0,1\}$ and data distribution $\D_h$, the loss or error of $f$ on $\D_h$ is defined as $L(f,\D_h)
= \expects{(x,y)\sim\D_h}{\id[f(x)\neq y]}$. For a finite multiset $T$ of labeled data we overload notation and denote $L(f,T) = L(f,\D_T) = \frac{1}{|T|} \sum_{(x,y) \in T} \id[f(x) \neq y]$. For marginals $\D,\D'$ on a countable space we write $\dtv(\D,\D')$ for total variation distance.  For
predictors $f,g$ and a marginal $\D$, write
$ \sd_\D(f,g) = \probs{x\sim\D}{f(x)\neq g(x)}$.
For $\gamma>0$, let the \emph{$\gamma$-packing number}
$ \packnum_\gamma(\H,\D) $ be the largest cardinality of a subset of $\H$ whose distinct members have pairwise $d_\D$-distance
strictly larger than $\gamma$.

\subsection{Relatively Smart Learning}A \emph{(fully-supervised) learner} $\A$ takes as input a sequence $S= (x_1,y_1), \ldots, (x_m,y_m)$ of labeled samples---often referred to as \emph{training data}---drawn i.i.d.~from $\D_h$, and outputs a predictor $\A(S): \X \to \{0,1\}$. We measure a learner's error as a function of the marginal~$\D$. This permits marginal by marginal comparison of learners catered to a marginal, referred to as \emph{$\D$-fixed learners}, to learners which are provided no such knowledge. 
\begin{definition}\label{def:err_rate}
  A \emph{distribution-dependent error rate} is a function $\epsilon: \Delta(\X) \times \NN \to [0,1]$, where $\epsilon(\D,m)$ is an error associated with a distribution $\D \in \Delta(\X)$ and a number of samples $m$.
\end{definition}

\begin{definition}\label{def:error_learner}
  For a learner $\A$, let its distribution-dependent error rate $\epsilon_\A(\D,m)$ be the worst-case, over hypotheses $h \in \H$, of its expected error with respect to $\D_h$ when given $m$ i.i.d.~samples from $\D_h$, i.e., \(\epsilon_\A(\D,m) = \sup_{h \in \H}  \expects{S \sim \D_h^m}{L\left(\A(S),\D_h\right)}.\) 
  \end{definition}

We recall the definitions of sound certifier, certifiable error rate, and relatively smart learning from \citet{pmlr-v336-dughmi26a}.

\begin{definition}[Sound certifier \citep{pmlr-v336-dughmi26a}]
Let $\A$ be a learner. We say a function $\C: \X^* \to [0,1]$ is a \emph{sound certifier} for $\A$ if for every distribution $\D$, $m \in \NN$, and $S \sim \D^m$ we have $\expects{}{\C(S)} \geq \epsilon_\A(\D,m)$.
\end{definition}

\begin{definition}[Certifiable error rate \citep{pmlr-v336-dughmi26a}]\label{def:certifiable_error_rate}
  A distribution-dependent error rate $\epsilon(.,.)$ is \emph{certifiable} if for every distribution $\D$, there exists a $\D$-fixed learner $\A$ and a sound certifier $\C$ for $\A$ such that for each $m \in \NN$ and $S \sim \D^m$, we have $\expects{}{\C(S)} \leq \epsilon(\D,m)$.
\end{definition}

\begin{definition}[Relatively Smart Learning \citep{pmlr-v336-dughmi26a}]
\label{def:relatively_smart}
For a function $\sigma: \NN \times (0,1)\to \NN$ and constant $\alpha > 0$,  we call a learner $\A$ \emph{relatively $(\alpha,\sigma)$-smart} if \[\epsilon_\A(\D,\sigma(m,\eta)) \leq \alpha \epsilon(\D,m) + \eta\] for every certifiable distribution-dependent error rate  $\epsilon$, every distribution  $\D$, every sample size $m \in \NN$, and every additive error parameter $\eta \in (0,1)$. We also say $\A$ is \emph{relatively smart} if there exist $\alpha$ and $\sigma$ such that it is relatively $(\alpha,\sigma)$-smart.
\end{definition}

We note that our negative results hold for countable domains, and our positive results hold more generally for domains, hypothesis classes, and distributions jointly satisfying standard measurability assumptions \citep[see e.g.][]{shalev-shwartz_understanding_2014}.
As is common in learning theory, we take a hands-off approach to the measure-theoretic details. 
\section{ERM is Relatively Smart}\label{sec:erm_smart}

We prove a positive resolution of Open Question~3.3 from \citet{pmlr-v336-dughmi26a}. We show that every proper consistent learner is relatively smart with $O(m^2)$ samples.

\begin{theorem}\label{thm:erm_smart}
For every domain and hypothesis class, we have that every proper consistent learner is relatively
$(128,\left\lceil
\frac{256(m+1)^2}{\eta}
\right\rceil)$-smart in the distribution-free setting. 
\end{theorem}

We formally prove Theorem~\ref{thm:erm_smart} in \ifthenelse{\boolean{submission}}{Appendix~\ref{sec:formal_pf_erm_smart}}{Section~\ref{sec:formal_pf_erm_smart}}, but the high-level idea is as follows. Fix a learner $\A_\D$ tailored to a marginal distribution $\D$, together with its sound certifier $\C$. For sample size~$m$, let $T$ consist of $m$ i.i.d.~draws from $\D$, and let  $\epsilon= \expect{\C(T)}$ be the certified error on $m$ samples. Now let $S$ be a multiset of $M=cm^2$ i.i.d.~draws from $\D$, for a sufficiently large constant $c$, and let $\D_S$ denote the uniform distribution over $S$. The first observation is that $\C$ is unable to reliably tell whether its input samples $T$ were drawn from $\D$ or from $\D_S$. The underlying intuition is the Birthday paradox, and this observation is used in the proof of Theorem~3.2 in \citet{pmlr-v336-dughmi26a} to show OIG is relatively smart. Therefore, soundness entails that $\A_\D$'s expected error on $\D_S$ is some  $\gamma=O(\epsilon)$, on average over $S$. 

We then show that if there is a learner that when given $m$ samples achieves error at most $\gamma$ on the uniform distribution $\D_S$, then there cannot be many hypotheses that are pairwise $4\gamma$-far from each other with respect to $\D_S$. In particular, we show that the $4\gamma$-packing number of $\H$ with respect to $\D_S$ is upper bounded by roughly $2^{m}$. This essentially follows from the fact that, for an unlabeled sample of size $m$, there are at most $2^m$ different labelings and therefore at most $2^m$ different predictors in the range of the learner. $\A_\D$ then implies the existence of a cover $\K$ of $\H$ with respect to $\D_S$ with $|\K| \approx 2^m$ and cover radius $4\gamma=O(\epsilon)$.

Finally, we use a standard double-sampling argument to bound the error of a consistent proper learner $\A_{\cons}$, such as ERM, on samples of size $M/2 = c m^2/2$.  Let $S$ be an unlabeled sample from $\D$ of size $M$ as before, partitioned uniformly at random into two equal halves $S_1$ and $S_2$, and fix a ground truth hypothesis $h^*$. Standard tail bounds and the union bound establish that every hypothesis in $\K$ has similar error (relative to $h^*$) on $\D_{S_1}$ and $\D_{S_2}$ --- within $O(\epsilon)$ with high probability. This extends to every hypothesis $h \in \H$; to see this, note that if $h$ had very different errors on $\D_{S_1}$ and $\D_{S_2}$, then so would its nearest neighbour in $\K$, at a distance of $O(\epsilon)$ with respect to $\D_S$. Now let $S_1$ be the sample given to $\A_{\cons}$ and let $h \in \H$ be its output, with $h{|S_1}=h^*{|S_1}$ by definition and therefore zero error on $D_{S_1}$. It follows that the error of $h$ on $\D_{S_2}$ is $O(\epsilon)$. Since the error on $\D_{S_2}$ is an unbiased estimate of the error on $\D$, this completes the argument. 

\ifthenelse{\boolean{submission}}{}{\ifthenelse{\boolean{submission}}{\section{Formal Proof of Theorem~\ref{thm:erm_smart}}\label{sec:formal_pf_erm_smart}}{\subsection{Formal Proof of Theorem~\ref{thm:erm_smart}}\label{sec:formal_pf_erm_smart}}

Let $\eps(.,.)$ be any certifiable error rate. Fix a distribution $\D$, and sample size $m$.  By definition of certifiable error rate, there exists a learner $\A = \A_{\D}$ that is catered to $\D$ that together with its sound certifier $\C = \C_{\D}$ witnesses the certifiable error rate $\eps(\D,m)$. In other words, $\A$ and $\C$ satisfy (i) $\epsilon_{\A}(\D,m) \leq \expects{S \sim \D^m}{\C(S)} \leq \epsilon(\D,m)$ and (ii) for any distribution $\D'$, $\epsilon_{\A}(\D',m) \leq \expects{S \sim {\D'}^m}{\C(S)}$.

We first show that $\eps(\D,m)$ can be turned into an upper bound on the error of $\A$ on the uniform distribution over a random sample from $\D$. This observation is used in the proof of Theorem~3.2 in \citet{pmlr-v336-dughmi26a} to show OIG is relatively smart. We include a proof here for completeness.

\begin{lemma}
\label{lemma:error_on_unif}
For any $M\ge m$, denote $\Gamma_{M,m} :=  \frac{M(M-1)\ldots (M-m+1)}{M^{m}}$. We have 
\[
\expects{S\sim\D^M}{\eps_\A(\D_S,m)}
\leq \Gamma_{M,m}\epsilon(\D,m)+1-\Gamma_{M,m}.
\]
\end{lemma}

\begin{proof}
Since $\C$ is a sound certifier for $\A$ we have by definition that $\epsilon(\D,m) \geq \expects{S\sim \D^m}{C(S)}$. Observe that drawing $m$ i.i.d.~samples from $\D$ is equivalent to drawing a larger multiset $S$ of size $M \geq m$ and then drawing a uniformly random multi-set $T \simnr \D_S^m$ of $m$ instances from $S$ \emph{without replacement}. Hence, we have $\expects{S\sim \D^m}{C(S)} = \expected{S \sim \D^M}\expects{T \simnr \D_{S}^{m}}{\C(T)}$. Finally, note that $\Gamma$ is the probability of observing no duplicates when making $m$ independent with replacement draws from a uniform distribution on $M$ points. Therefore, 
\[
\expected{S \sim \D^M}\expects{T \simnr \D_{S}^{m}}{\C(T)}
        \geq  \frac{1}{\Gamma}\expected{S \sim \D^M}{\expects{T \sim \D_{S}^{m}}{\C(T)}} + 1- \frac{1}{\Gamma}
         \geq  \frac{1}{\Gamma}\expects{S \sim \D^M}{\epsilon_{\A}(\D_S,m)}+1- \frac{1}{\Gamma},
\]
concluding the proof.
\end{proof}

We now show that the bound on learner's error with respect to the uniform distribution over a random sample, can be turned into a bound on the packing number of such distributions. Particularly, we show that a bounded error on \emph{any} distribution $\Q$ implies a bound on the packing number of the class with respect to $\Q$.
\begin{lemma}\label{lemma:packing-error}
Let $\Q$ be any marginal distribution. If
$\packnum_\gamma(\H,\Q)\ge 2^{m+1}$, then for every (randomized) learner $\B$ we have
\[
\eps_\B(\Q,m)\ge\frac{\gamma}{4}.
\]
\end{lemma}

\begin{proof}
Fix a learner $\B$. Choose $K=2^{m+1}$ hypotheses $h_1,\ldots,h_K$ that are pairwise more than $\gamma$ apart under $\Q$, i.e., $\forall i,j \in [K], \sd_{Q}(h_i,h_j) > \gamma$.
Let the target hypothesis be chosen uniformly at random from $[K]$. For a multi-set $X=(x_i)_{i=1}^m$, we also denote by $X_J=(x_i,h_J(x_i))_{i=1}^m$ the mulit-set that is obtained by labeling $X$ with $h_J$. Moreover, we will abuse the notation and write $Q_J$ to denote the joint distribution of $(x,h_J(x))$ with $x\sim Q$.

We can write the expected error of the learner $\B$ as 
\[
\expected{J\sim [K]}\expects{X \sim Q^m}{L(\B(X_{J}),Q_{J})} = \expected{X \sim Q^m}\expects{J\sim [K]}{\sd_{Q}(\B(X_{J}),h_J)}
\]

Fix a set $X$. Since $X$ has size $m$, there are at most $R\leq 2^m$ many distinct labelings that $h_1,\ldots,h_K$ can induce on $X$. Let $(A_i)_{i=1}^R$ denote the equivalence classes induced by the labelings. Observe that for any $A_i$ the output of the learner is fixed and can be within distance of $\gamma/2$ of at most one hypotheses in $A_i$. Otherwise, by the triangle inequality there would exist two hypotheses $h,h'\in A_i$ with $\sd_{Q}(h,h')\leq \gamma$, contradicting the assumption that $h_1,\ldots,h_k$ are pairwise $\gamma$-apart. Hence, 

\[
\expects{J\sim [K]}{\sd_{Q}(\B(X_{J}),h_J)} > \frac{1}{K} \sum_{i \in R} (|A_i| -1)\gamma/2 = \frac{\gamma}{2K} \sum_{i \in R} (|A_i| -1) =\frac{\gamma(K-R)}{2K} \geq \frac{\gamma}{4}.
\]

This proves that for any (randomized) learner $\B$ we have that $\expected{J\sim [K]}\expects{X \sim Q^m}{L(\B(X_{J}),Q_{J})} \geq \gamma/4$. Therefore, for any (randomized) $\B$ there exists $h_{i^*}$ with $\expects{X \sim Q^m}{L(\B(X_{i^*}),Q_{i^*})} \geq \gamma/4$, completing the proof. 
\end{proof}

We have so far argued that from the certifiable error rate $\epsilon(\D,m)$ we can conclude an upper bound on the error of $\A$ on $\D_S$ when given $m$ samples, on average over random $S\sim \D^M$ of size $M$. We then proved that bound on the error of $\A$ on $\D_S$ can be used to control the packing number of the class with respect to $\D_S$. We now use a double-sampling argument to show that the bound on the packing number with respect to $\D_S$ can be used to bound the error of a consistent hypothesis with $M/2$ samples. The following lemma formalizes this claim and its proof can be found in Appendix~\ref{app:pf_erm_smart}.

\begin{lemma}\label{lemma:random-split}
Let $\gamma\in(0,1)$, and $n$ be any integer with $n\gamma\geq 8\ln2$.  For every integer $K\ge1$ and $h'\in\H$,
\[
\probs{S\sim\D_{h'}^n}{\{\exists h \in \H: L(h,S) = 0 ,\, L(h,\D_{h'}) \geq \gamma \}}
\leq
2\probs{Z\sim\D^{2n}}{
\packnum_{\gamma/16}(\H,\D_Z)\ge K
}
+
2K\exp\!\left(-\frac{n\gamma}{96}\right).
\]
\end{lemma}

We are now ready to finish the proof. For any consistent learner $\A_{\cons}$, we can write its expected error when given samples of size $n$ as

\begin{equation*}
\begin{aligned}
\epsilon_{\A_{\cons}}(\D,n) &= \sup_{h'\in \H}\expects{S\sim \D_{h'}^n}{L(\A_{\cons}(S),\D_{h'})} \\
&\leq \sup_{h'\in \H}\int_{0}^{1}\probs{S\sim \D_{h'}^n}{\exists h \in \H: L(h,S) = 0 ,\, L(h,\D_{h'}) \geq \gamma }d\gamma.
\end{aligned}
\end{equation*}

Now let $K = 2^{m+1}$ and $\tau = \frac{96\ln 2K}{n} = \frac{96 (m+2)\ln 2}{n}$. We can verify that $n\tau \geq 8\ln 2$. Therefore, applying Lemma~\ref{lemma:random-split}, we get that 
\begin{equation}\label{eq:error_consistent1}
\begin{aligned}
    \epsilon_{\A_{\cons}}(\D,n) &\leq \tau + \sup_{h'\in \H} \int_{\tau}^{1}\probs{S\sim \D_{h'}^n}{\exists h \in \H: L(h,S) = 0 ,\, L(h,\D_{h'}) \geq \gamma }d\gamma\\
    & \leq \tau + 2\int_{\tau}^1 \probs{Z\sim\D^{2n}}{
\packnum_{\gamma/16}(\H,\D_Z)\geq 2^{m+1}}\,d\gamma + 2\int_{\tau}^1 2^{m+1}\exp\!\left(-\frac{n\gamma}{96}\right)\,d\gamma
\end{aligned}
\end{equation}
From Lemma~\ref{lemma:packing-error}, we know that for any $\gamma\geq 0$, we have 
\[\probs{Z\sim\D^{2n}}{
\packnum_{\gamma/16}(\H,\D_Z)\geq 2^{m+1}} \leq \probs{Z\sim\D^{2n}}{\epsilon_{\A}(\D_Z,m)\geq \gamma/64}.
\]
Therefore,
\[
\begin{aligned}
2\int_{\tau}^1 \probs{Z\sim\D^{2n}}{
\packnum_{\gamma/16}(\H,\D_Z)\geq 2^{m+1}} &\leq 2\int_{0}^1\probs{Z\sim\D^{2n}}{\epsilon_{\A}(\D_Z,m)\geq \gamma/64}\\
&= 128\int_{0}^{1/64}\probs{Z\sim\D^{2n}}{\epsilon_{\A}(\D_Z,m)\geq \beta}d\beta\\
& \leq 128 \expects{Z\sim \D^{2n}}{\epsilon_{\A}(\D_Z,m)}.
\end{aligned}
\]
Moreover, 
\[ 2^{m+2}\int_\tau^1\exp\left(-\frac{n\gamma}{96}\right)\,d\gamma
\leq
\frac{96\cdot 2^{m+2}}{n}\exp\left(-\frac{n\tau}{96}\right)
=
\frac{96}{n}.
\]
Combining the above with Equation~\eqref{eq:error_consistent1}, we get
\begin{equation}\label{eq:error_consistent2}
\epsilon_{\A_{\cons}}(\D,n) \leq 128 \expects{Z\sim \D^{2n}}{\epsilon_{\A}(\D_Z,m)} + \frac{96}{n} + \frac{96 (m+2)\ln 2}{n}.
\end{equation}
We now apply Lemma~\ref{lemma:error_on_unif} with $M = 2n$, to conclude
\[
\expects{Z\sim\D^{2n{}}}{\eps_\A(\D_Z,m)}
\leq \Gamma_{2n,m}\epsilon(\D,m)+1-\Gamma_{2n,m} \leq \epsilon(\D,m) + \frac{m(m-1)}{4n},
\]
where we used the fact that $1-\Gamma_{2n,m} \leq m(m-1)/4n$. Therefore, we can continue Equation~\eqref{eq:error_consistent2} and write
\[
\epsilon_{\A_{\cons}}(\D,n) \leq 128 \epsilon(\D,m) + \frac{32m(m-1)}{n}+\frac{96 + 96 (m+2)\ln 2}{n} \leq 128 \epsilon(\D,m) + \frac{256(m+1)^2}{n},
\]
where the last inequality is due to $32m(m-1) + 96 + 96(m+2)\ln 2 \leq 256(m+1)^2$.
Letting $n \geq \frac{256(m+1)^2}{\eta}$, we get that
\[
\epsilon_{\A_{\cons}}(\D,n) \leq 128 \epsilon(\D,m) + \eta,
\]
concluding the proof. \hfill \qed}

\section{Relatively Smart Semi-Supervised Learners}\label{sec:ss_smart}
We know that the ERM and OIG learners are relatively smart with $\sigma(m,\eta) = O_{\eta}(m^2)$ (Theorem~\ref{thm:erm_smart} in this paper and Theorem~3.2 from \citet{pmlr-v336-dughmi26a}). Moreover, we know this is essentially optimal for \emph{fully supervised} learners: for every pair of constants $\alpha$ and $\beta$ and every function $\sigma(m,\eta) = O_\eta (m^{2 - \beta})$, no fully supervised learner is relatively $(\alpha,\sigma)$-smart (Theorem~4.1 of \citet{pmlr-v336-dughmi26a}). This shows that a quadratic blowup in sample complexity is unavoidable when data can only come with labels. We therefore move beyond the supervised setting and ask whether access to unlabeled data can reduce the number of labeled samples required for relatively smart learning.

\begin{center}
\begin{minipage}{0.9\linewidth}
\emph{Are there semi-supervised learners that are relatively smart with $\sigma(m,\eta)$ significantly better than $O_{\eta}(m^2)$? If so, can such improvements be achieved by tractable semi-supervised learners?}
\end{minipage}
\end{center}

We answer the first question in the affirmative. We show that there exists a semi-supervised relatively smart learner that uses $O_{\eta}(m^2)$ unlabeled samples only $m$ of which are labeled---i.e., it achieves $\sigma(m,\eta) = m$. However, this learner is highly intractable, and there is no general recipe for implementing it efficiently. Indeed, we obtain a strong negative answer for the second question. We prove that, even if the marginal distribution is given, for every pair of constants $\alpha$ and $\beta$ and every function $\sigma(m,\eta) = O_\eta (m^{2 - \beta})$, no semi-supervised learner can be relatively smart  with labeled sample complexity blowup $\sigma(m,\eta)$ if it only makes polynomially many (in $m$) calls to an agnostic ERM oracle.

We formally define relatively smart learning for semi-supervised learners as follows.

\begin{definition}\label{def:error_semi-supervised}
  For a semi-supervised learner $\A$, we define its distribution-dependent error rate $\epsilon^{\sss}_\A(\D,u,\ell)$ be the worst-case, over hypotheses $h \in \H$, of its expected error with respect to $\D_h$ when given $u$ i.i.d.~unlabeled samples from $\D$ and $\ell$ i.i.d.~labeled samples from $\D_h$, i.e., \(\epsilon^{\sss}_\A(\D,u,\ell) = \sup_{h \in \H}  \expects{U \sim \D^u, S \sim \D_h^\ell}{L\left(\A(U,S),\D_h\right)}.\) 
  \end{definition}

  \begin{definition}[Relatively Smart Semi-Supervised Learning]
\label{def:ss_relatively_smart}
For functions $\sigma_u,\sigma_{\ell}: \NN \times (0,1)\to \NN$ and constant $\alpha > 0$,  we call a semi-supervised learner $\A$ \emph{relatively $(\alpha,\sigma_u,\sigma_{\ell})$-smart} if \[\epsilon^{\sss}_\A(\D,\sigma_u(m,\eta),\sigma_\ell(m,\eta)) \leq \alpha \epsilon(\D,m) + \eta\] for every certifiable distribution-dependent error rate  $\epsilon$, every distribution  $\D$, every sample size $m \in \NN$, and every additive error parameter $\eta \in (0,1)$.
\end{definition}

\subsection{A Relatively Smart Semi-Supervised Learner with No Blow-Up in Labeled Sample Complexity}

We construct a learner that is relatively semi-supervised smart with no blow-up in labeled sample complexity, i.e., with $\sigma_\ell(m,\eta) = m$, \emph{independent of $\eta$.} The construction is inspired by the idea that certifiable error rates cannot be much smaller than the error of the learner that is witnessing it on uniform distributions over a random sample. This idea  appears in Lemma~\ref{lemma:error_on_unif} and is also exploited by \citet{pmlr-v336-dughmi26a} to prove that OIG is relatively smart. 

In more detail, we show that a certifiable error rate must be close to the error of the optimal distribution-fixed learner for the uniform distribution over a random unlabeled sample of size $O(m^2)$, when allowed $O(m)$ labeled samples. This shows that for functions $\sigma_u(m,\eta) = O_{\eta}(m^2)$, $\sigma_\ell(m,\eta) = O(m)$, the semi-supervised learner $\A_{\sss}$ defined as follows is relatively smart: given unlabeled data $U$, labeled data $S$, and test point $x$, $\A_{\sss}$ computes an optimal learner for the uniform distribution over the pooled sample $T = U \cup \dom(S) \cup x$, runs this learner on $S$ and outputs its prediction on $x$.
Note that $\H|T$ is finite and the ``optimal learner'' is well-defined. This learner is described in Algorithm~\ref{alg:ss-learner} and its guarantee is summarized in the following theorem.

\begin{algobox}{Learner $\A_{\sss}(U,S)(x)$ trained on unlabeled and labeled samples $U,S$ given test point $x$:}\label{alg:ss-learner}
    \begin{enumerate}[itemsep=0pt]
    \item Let $T = U \cup \dom(S) \cup x$. Let $\A^*_{T}$ be an optimal \emph{supervised} learner for the uniform distribution on $T$ when given $\ell:=|S|$ labeled samples, i.e.,
    \( \A^*_{T}\in \arg\min_{\A} \epsilon_{\A}(\D_T,\ell)
    \). 

    \item Return $\A^*_T(S)(x)$.
\end{enumerate}
\end{algobox}

\begin{theorem}\label{thm:ss-smart-learner}
    For every domain and hypothesis class, the semi-supervised learner $\A_{\sss}$ defined in Algorithm~\ref{alg:ss-learner} is relatively $(1,\lceil\frac{2m(m+1)}{\eta}\rceil,m)$-smart in the distribution-free setting.
\end{theorem}

\begin{proof}
    Let $\epsilon(.,.)$ be a certifiable error rate. Fix a distribution $\D$ and sample size $m$. For a finite unlabeled multiset $U \in \X^*$, we will use the shorthand notation $U_h$ to refer to $h|U$. We write the expected error of $\A_{\sss}$ on $\D$ when given $u$ unlabeled and $\ell$ labeled samples as 
    \begin{equation}\label{eq:error_ss_wor}
    \begin{aligned}
   \epsilon^{\sss}_{\A_{\sss}}(\D,u,\ell) &= \sup_{h \in \H}  \expected{U \sim \D^u} \expects{S \sim \D_h^\ell}{L\left(\A_{\sss}(U,S),\D_h\right)} \\
   &= \sup_{h \in \H}  \expected{U \sim \D^u}\expected{ S \sim \D_h^\ell}\expects{x \sim \D}{\id\left\{\A_{\sss}(U,S)(x) \neq h(x)\right\}}\\
   &=\sup_{h\in\H} \expected{U \sim \D^u}\expected{ S \sim \D_h^\ell}\expects{x \sim \D}{ \id\left\{ \A^{*}_{U \cup \dom(S) \cup x}(S)(x)\neq h(x) \right\} }\\
   &=\sup_{h\in\H} \expected{T\sim \D^M} \expected{ S \simnr \D_T^{\ell+1}}\expects{i\sim [\ell+1]}{ \id\left\{ \A^{*}_T(S_h^{(-i)})(x_i)\neq h(x_i) \right\} },
    \end{aligned}
    \end{equation}
where $x_i$ denotes the instance in $S$ at index $i$ and $S_h^{(-i)}$ is the set $S$ without $x_i$. Moreover, $\simnr$ denotes sampling without replacement.

Let $M = u + \ell +1$. Now observe that the process of sampling $S$ without replacement from $\D_T$ and then choosing $x$ uniformly from $S$ is the same as the sampling $\ell$ i.i.d.~samples from $\D_T$ labeled by $h$ as the training set and an additional i.i.d.~sample from $\D_T$ as the test point, save for the fact that there are no duplicates among $S$ and $x$. Recalling that $\Gamma_{M,\ell+1} = \frac{M(M-1)\ldots (M-\ell)}{M^{\ell+1}}$ is the probability of observing no duplicates when making $\ell+1$ independent draws from a uniform distribution on $M$ points, we get from Equation~\eqref{eq:error_ss_wor} that
 \[
    \begin{aligned}
   \epsilon^{\sss}_{\A_{\sss}}(\D,u,\ell) & \leq \frac{1}{\Gamma_{M,\ell+1}}\sup_{h\in\H} \expected{T\sim \D^M} \expected{S\sim\D_T^\ell}\expects{x\sim \D_T}{ \id\left\{ \A^{*}_T(S_h)(x)\neq h(x) \right\} }\\
   & \leq \frac{1}{\Gamma_{M,\ell+1}}\expected{T\sim \D^M} \sup_{h\in\H} \expected{S\sim\D_T^\ell}\expects{x\sim \D_T}{ \id\left\{ \A^{*}_T(S_h)(x)\neq h(x) \right\} }\\
    & = \frac{1}{\Gamma_{M,\ell+1}}\expected{T\sim \D^M}  \epsilon_{\A^{*}_T}(\D_T,\ell).\end{aligned}\]

Since $\epsilon(.,.)$ is certifiable, by definition, there exists a learner $\A = \A_\D$ and its sound certifier $\C= \C_\A$ such that (i) $\epsilon_{\A}(\D,m) \leq \expects{S \sim \D^m}{\C(S)} \leq \epsilon(\D,m)$ and (ii) for any distribution $\D'$, $\epsilon_{\A}(\D',m) \leq \expects{S \sim {\D'}^m}{\C(S)}$. Letting $\ell = m$, similar to the proof of Lemma~\ref{lemma:error_on_unif}, we can conclude that 
\[
\expects{T\sim\D^M}{\eps_\A(\D_T,m)}
\leq \Gamma_{M,m} \cdot \epsilon(\D,m)+1-\Gamma_{M,m}.
\]

We know $\A^*_{T}$ is the optimal (supervised) distribution-fixed learner for $\D_T$ when given $m$ labeled samples. Therefore, dataset by dataset, we get that $\epsilon_{\A^{*}_T}(\D_T,m) \leq \epsilon_{\A}(\D_T,m)$. Hence, $\frac{1}{\Gamma_{M,m+1}}\expects{T\sim \D^M} { \epsilon_{\A^{*}_T}(\D_T,m)} \leq \frac{1}{\Gamma_{M,m+1}}\expects{T\sim \D^M} {\epsilon_{\A}(\D_T,m)}$ and we get from above that 
 \[
    \begin{aligned}
   \epsilon^{\sss}_{\A_{\sss}}(\D,u,m) & %\leq \frac{1}{\Gamma_{M,m+1}}\expects{T\sim \D^M} { \epsilon_{\A^{*}_T}(\D_T,m)} 
   \leq \frac{1}{\Gamma_{M,m+1}}\expects{T\sim \D^M} {\epsilon_{\A}(\D_T,m)} \leq \frac{\Gamma_{M,m}}{\Gamma_{M,m+1}}\epsilon(\D,m) + \frac{1 -\Gamma_{M,m}}{\Gamma_{M,m+1}}.\end{aligned}\]

  Let $u = \frac{2m(m+1)}{\eta}$ and thus $M = \frac{2m(m+1)}{\eta} + m +1$. We know that $\Gamma_{M,m+1} \geq 1-\frac{m(m+1)}{2M} \geq 1-\frac{m(m+1)}{2u} \geq 1-\eta/4$.
   Therefore, we get that
   \[
   \epsilon^{\sss}_{\A_{\sss}}(\D,u,m) \leq \frac{\epsilon(\D,m)+\eta/4}{1-\eta/4} \leq \epsilon(\D,m) + \eta,
   \]
   where the last inequality is due to $\eta\in (0,1)$. This completes the proof.
\end{proof}
\subsection{An Inclusion Graph Relatively $(1,O(m^2/\eta),m)$-Smart Learner}
In the proof of Theorem~\ref{thm:ss-smart-learner}, we formulated Algorithm~\ref{alg:ss-learner} in terms of an optimal distribution-fixed learner for the uniform distribution $\D_T$, without specifying the form of this learner. We now show that essentially the same guarantee can be achieved by a natural generalization of the OIG learner. The key observation is that the semi-supervised construction naturally induces a \emph{leave-most-out} transductive problem: from a finite set $T$, the learner observes the labels of a small subset of $m$ points and must predict the labels of the remaining points. We define an Inclusion Graph (IG) learner for the leave-most-out problem and observe that it is optimal in this setting.

The inclusion graph is defined for a set $T$ consisting of $M = O(m^2/\eta)$ unlabeled data points. The nodes are the behaviors $\H|T$ of $\H$ on $T$. For every subset $S$ of $T$ of size $m$ and every realizable labeling $h_S \in \H|S$ on $S$, we include a hyper-edge between all nodes consistent with $h_S$. Any learner for the transductive leave-most-out setting should ``orient'' every such edge by selecting a \emph{completion} of $h_S$ in $\{0,1\}^{T}$. Note that while this completion can be taken to be consistent with $h_S$ without loss, inducing $2^{M-m}$ possible choices, it need not be consistent with any single behavior in $\H|T$---i.e., it need not be ``proper''. The out-degree of a node is then the average over all its edges of the hamming distance between the node and the completion chosen for that edge. It is easy to observe that the leave-most-out error of every node equals it out-degree divided by $M-m$. An $\ig_{T}$ learner for the set $T$ is then any (randomized) orientation that minimizes the maximum (expected) out-degree. We can easily see that such a learner achieves the optimal leave-most-out error. This is reminiscent of the fact that the OIG learner is optimal for the leave-one-out transductive setting. 

The learner $\A_{\sss-\ig}$ is then defined as follows: given unlabeled data $U$, labeled data $S$, and test point $x$, predict according to the optimal orientation of $\ig_{T}$ for $T = U \cup \dom(S) \cup x$. We have 
\[\begin{aligned}
     \epsilon^{\sss}_{\A_{\sss-\ig}}(\D,u,\ell) &\leq \expected{T\sim \D^M} \sup_{h\in\H} \expected{S\simnr\D_T^{\ell+1}}\expects{i\sim [\ell+1]}{ \id\left\{ \ig_T(S_h^{(-i)})(x_i)\neq h(x_i) \right\}} \\
     &\leq \expected{T\sim \D^M} \sup_{h\in\H} \expected{S\simnr\D_T^{\ell+1}}\expects{i\sim [\ell+1]}{ \id\left\{ \A(S_h^{(-i)})(x_i)\neq h(x_i) \right\}}\\
     &\leq \frac{1}{\Gamma_{M,\ell+1}}\expects{T\sim \D^M} {\epsilon_{\A}(\D_T,\ell)},
\end{aligned}
\]
where the  second inequality is due to the optimality of the inclusion graph learner in the leave-most-out setting.
Much like in Theorem~\ref{thm:ss-smart-learner}, we can continue to prove that $ \epsilon^{\sss}_{\A_{\sss-\ig}}(\D,O(m^2/\eta),m) \leq \epsilon(\D,m) +\eta$, concluding that $\A_{\sss-\ig}$ is relatively $(1,O(m^2/\eta),m)$-smart.

\section{An Impossibility Result for Tractable Semi-Supervised Learners}\label{sec:oracle_impossibility}

The preceding subsection shows that access to unlabeled data can eliminate the quadratic blow-up in labeled sample complexity: a relatively smart semi-supervised learner can use $O_\eta(m^2)$ unlabeled samples while requiring only $m$ labeled samples. We now show that this comes at the cost of the tractability of the learner.

We formalize tractability through \emph{agnostic ERM oracle} access to the hypothesis class. An agnostic ERM oracle for a hypothesis class $\H$ takes as input a finite labeled multi-set and returns a hypothesis in $\H$ minimizing its error on the set.

\begin{definition}[Agnostic ERM Oracle]
    For a hypothesis class $\H \subseteq \{0,1\}^{\X}$, an Agnostic ERM Oracle $\O_{\H}$ for $\H$ is a function that given any finite labeled multiset $T \in (\X\times\{0,1\})^*$, outputs a hypothesis in $\H$ with the smallest error on $T$, i.e., $\O_{\H}(T) \in \arg\min_{h\in \H} L(h,T)$.
\end{definition}

We only consider semi-supervised learners which do not have direct knowledge of $\H$, with their only access to $\H$ being through calls to such an oracle. Formally, a (randomized) algorithm $\A$ is said to be an oracle learner if it takes as input an unlabeled set $U$ (or the marginal distribution $\D$ if granted), a labeled set $S$ of size $m$, and a test point $x$, adaptively queries an agnostic ERM oracle, then outputs a label for $x$. We  say that $\A$ is an efficient oracle learner if the number of oracle calls it makes is bounded by a polynomial function of $m$.
We show that no efficient oracle learner can be relatively smart with a substantially sub-quadratic~$\sigma_\ell$.

\begin{theorem}\label{thm:oracle_impossibility}
    There exists a countable domain such that for every pair of constants $\alpha,\beta$, any arbitrary $\sigma_u$, and every function $\sigma_{\ell}(m,\eta) = O_{\eta}(m^{2-\beta})$ the following holds: there is no (semi-supervised) oracle efficient learner $\A$ that is relatively $(\alpha,\sigma_u,\sigma_{\ell})$-smart for every hypothesis class. Particularly, this holds even if the learner has full knowledge of the marginal distribution.
\end{theorem}

Theorem~\ref{thm:oracle_impossibility} is formally proved in \ifthenelse{\boolean{submission}}{Appendix~\ref{sec:formal_pf_oracle_impossibility}}{Section~\ref{sec:formal_pf_oracle_impossibility}}. We now give the high-level idea of the proof. Our construction builds on the lopsided hypothesis classes used in the proof of the lower bound for OIG and ERM learners in \citet{pmlr-v336-dughmi26a}. On a domain of size $n$, the lopsided class $\H(n)$ is defined as follows:  consider all hypotheses for which at most $n^{1-O(\beta)} = o(n)$ of points have the same minority label and the remaining $n-o(n)$ points have the other majority label. We augment the class $\H(n)$ with a uniformly random hypothesis $f \in \{0,1\}^n$. For a (nearly) uniform distribution over the $n$ points, when the class (and hence the hypothesis $f$) is explicitly given, empirical risk minimization restricted to $f$ and the two constant predictors (all-$0$ and all-$1$) has vanishing error $o(n)/n = o(1)$. 
Moreover, when $m \approx \sqrt{n}$ this error rate is certifiable by almost the same arguments as in \citet{pmlr-v336-dughmi26a}: $m \approx \sqrt{n}$ is sufficiently large for a certifier to identify such a (nearly) uniform distribution. The intuition comes from the Birthday paradox, and to make it formal we rely on uniformity testing. 

A learner with only oracle access to the class, however, must somehow discover the additional hypothesis $f$. To show that this is difficult, we prove that, on average over the choice of $f$, the learner has error $1/2-o(1)$ when $f$ is the labeling function and the distribution is uniform on the domain. When the labeled sample size is smaller than half of the number of minority labels, the sample reveals only a small fraction of the labels of $f$, leaving the remaining labels uniformly random. We consider an agnostic ERM oracle that, whenever a hypothesis from the lopsided class $\H(n)$ is an empirical-risk minimizer, breaks ties in favor of $\H(n)$. For this oracle, Agnostic ERM queries reveal essentially no additional information about $f$. Indeed, for any fixed query, the fraction of hypotheses $f\in\{0,1\}^n$ that have strictly smaller empirical loss than every hypothesis in the lopsided class $\H(n)$ is exponentially small. To see this, we first show that any two such hypotheses $f,f'$ must agree on at least $n^{1-O(\beta)}$ points. Otherwise, consider the two hypotheses that agree with $f,f'$ on their agreement set and are respectively constant $0$ and $1$ on their disagreement set. Both belong to $\H(n)$, and one has loss no larger than $f$ or $f'$, contradicting the assumption that both strictly outperform every hypothesis in $\H(n)$. A combinatorial bound then shows that any family of hypotheses with this much pairwise agreement constitutes only an exponentially small fraction of $\{0,1\}^n$.

The same conclusion continues to hold for polynomially-many adaptively chosen queries. Fix the labeled sample, of size smaller than half the number of minority labels $n^{1-O(\beta)}$,
and fix also the learner's internal randomness.
Consider the learner's execution path, and in particular its sequence of oracle queries and responses. So long as the oracle has not already returned $f$, the principle of deferred decisions implies that the labels of $f$ outside the sample remain uniformly random. The next query to the oracle, therefore, is fixed independently of these unseen labels of $f$, and by the previous argument the oracle returns $f$ with only exponentially small probability. Carrying on inductively, a union bound  over polynomially many oracle calls shows that the oracle returns $f$ with probability $o(1)$, taken over the uniformly random choice of $f$. 

In the high probability event that $f$ is never output by the oracle, the learner’s output is a fixed predictor determined independently of the unseen labels of $f$. Over the uniformly random choice of $f$, this predictor is bound to err with probability $1/2$ on all test points outside the $o(n)$-sized sample. Therefore, the learner's expected error, averaged over the random choice of $f$, is at least $1/2 - o(1)$ for every fixed sample and fixed choice of the learner's internal randomness. By the probabilistic method, there must exist a fixed hypothesis $f^*\in \{0,1\}^n$ on which the learner has expected error at least $1/2-o(1)$, where expectation is over samples and the learner's internal randomness. 

Recall that the certifiable error rate is $o(1)$ when the number of samples is on the order of $\sqrt{n}$, whereas  the above argument holds for every sample size bounded by the number of minority labels $n^{1-O(\beta)}$. Taking $\beta$ arbitrarily  small therefore rules out oracle-efficient relatively smart learning with substantially sub-quadratic blowup in labeled sample complexity. Since we fixed the uniform  distribution on the domain, this holds even for oracle-efficient learners that are provided the marginal.

\ifthenelse{\boolean{submission}}{}{\ifthenelse{\boolean{submission}}{\section{Formal Proof of Theorem~\ref{thm:oracle_impossibility}}\label{sec:formal_pf_oracle_impossibility}}{\subsection{Formal Proof of Theorem~\ref{thm:oracle_impossibility}}\label{sec:formal_pf_oracle_impossibility}}
Let $\A$ be any oracle efficient learner. Fix an arbitrary parameter $\beta \in (0,\frac{1}{8})$. Consider the domain $\X \subseteq \bigcup_{n\in \NN} \X_n$ with $\X_n = \{n\} \times [n]$. We refer to $\X_n$ as the $nth$ row of the domain. For every $n\in \NN$, let $\D(n)$ denote the uniform distribution over $\X_n$. For every $n\in \NN$, we begin with the class $\H(n)$ constructed by \citet{pmlr-v336-dughmi26a}, where the minority label of every $h\in \H(n)$ has probability at most $n^{-\beta}$ under $\D(n)$. We then augment $\H(n)$ by adding random functions on $\X_n$. Define the functions $M,m:\NN \rightarrow \NN$ and $\xi: \rightarrow (0,1)$ as 
\begin{equation}\label{eq:m(n)xi(n)}
M(n) = \lceil n^{1-\beta}\rceil,
m(n) = \lfloor n^{\frac{1}{2}+3\beta} \rfloor, \,\text{and}\, \xi(n) = M(n)/n \approx n^{-\beta}.
\end{equation}
For $b\in\{0,1\}$ and $n \in \NN$, let 
\[
\begin{aligned}
&\H^{(b)}(n) = \{h \in \{0,1\}^{\X}: \forall x\in \X\setminus \X_{n}, h(x)=1 \,\text{and}\,|\{x\in\X_{n}: h(x)=b\}| \leq M(n)\},
\end{aligned}
\]
and define $\H(n) = \H^{(0)}(n) \cup \H^{(1)}(n)$.  
 Note that the function $M(n)$ is an upper bound on the number of points with minority label of every $h\in \H(n)$ in row $n$, while $h$ is constant on $\X \setminus \X_n$. Moreover, under the distribution $\D(n)$, the function $\xi(n)$ denotes an upper bound on the probability of a minority label for every $h\in\H(n)$.  Now we will define a hypothesis class that is a union of $\H(n)$ with (at most) one extra labeling for every row $n$. For any $g_n\in\{0,1\}^n$, define $(f_{n,g_n})_{n\in \NN}$ with $f_{n,g_n}(x) = g_n(x)$ for all $x \in \X_n$ and $f_{n,g_n}(x) = 1 $ for all $x\in \X \setminus \X_n$. Let
 $\H(n,g_n)= \H(n) \cup \{f_{n,g_n}\}$ and for any choice of binary labelings $G = (g_n)_{n\in \NN}$ with $g_n\in\{0,1\}^n$, define
 $\HH(G) = \bigcup_{n \in \NN}\H(n,g_n)$.
We show that for every $G$, the error rate $O(\xi(n))$ is certifiable for the class $\HH(G)$ with respect to $\D(n)$ and sample sizes $m\geq m(n) \approx n^{1/2+3\beta}$.

Fix some $\HH(G)$. Recall that for every $h\in \H(n)$, predicting its majority label has error at most $O(\xi(n))$ under distributions close to $\D(n)$ whereas all $h\notin \H(n,g_n)$ are constant $1$ on row $n$. Thus, when the learner has access to $\HH(G)$ and knows $G$, a natural strategy for learning with respect to nearly uniform distributions on row $n$ is to choose between the majority label of the sample and the label induced by $f_{n,g_n}$. We therefore construct the learner $\A_{n,G}$ (Algorithm~\ref{alg:s-fixed-learner}) that validates between $f_{n,g_n}$ and the two constant hypotheses $h_\mathbf{1}, h_\mathbf{0}$ defined as $h_\mathbf{1}(x) = 1$ and $h_\mathbf{0}(x) = 0$ for all $x \in \X$. We show that for all $m\geq m(n)$, the learner $\A_{n,G}$ has error $O(\xi(n))$ under $\D(n)$ and, importantly, this error rate is certifiable.

\begin{lemma}\label{lemma:certifiable_A_f_error}
    There exists an absolute constant $C = C(\beta)$ such that for every $n\geq C(\beta)$ and $G$ the following holds: for the hypothesis class $\HH(G)$, the rate
    \[\epsilon_{n}(\D,m) = \begin{cases}
    6\xi(n), & \text{if $\D = \D(n)$ and $m \geq m(n)$}\\
    1 & \text{otherwise}
\end{cases}\]
is certifiable for $\A_{n,G}$ in the distribution-free setting.
\end{lemma}
We prove the above lemma certifying the claimed error rate for $\A_{n,G}$. Note that $\A_{n,G}$ is specifically designed to perform well on the distributions that are close to being uniform on row $n$ but may incur large error on distributions far from $\D(n)$. Consequently, a sound certifier witnessing the claimed error rate must output small values only when the distributions is close to uniform on row $n$. A certifier with precisely this guarantee is developed in the proof of Theorem~3.1 of \citet{pmlr-v336-dughmi26a}, where uniformity testing is exploited to distinguish the uniform distribution on row $n$ from distributions that are substantially far from it in total variation distance. We therefore use this construction to define the certifier for $\A_{n,G}$. In particular, for each row $n$, we define the certifier $\C_n$ (Algorithm~\ref{alg:certifier_majority}) that relies on the modified uniformity tester $\tunif$ (Algorithm~\ref{alg:uniformity test}) developed in \citet{pmlr-v336-dughmi26a}, which adapts a standard uniformity tester to also detect distributions with support outside the $nth$ row. We provide the $\tunif$ (Algorithm~\ref{alg:uniformity test}) in Appendix~\ref{app:formal_pf_oracle_impossibility} and the remaining technical arguments for the proof of Lemma~\ref{lemma:certifiable_A_f_error} in Appendix~\ref{app:pf_certifiable_A_f_error}.

\begin{algobox}{Learner $\A_{n,G}(S)(x)$ trained on sample $S$ given test point $x$:}\label{alg:s-fixed-learner}
    \begin{enumerate}[itemsep=0pt]
      \item Let $\hat{h} \in \argmin_{h \in \{f_{n,g_n},\hone,\hzero\}} L(h,S)$.
    \item Return $\hat{h}(x)$.
\end{enumerate}
\end{algobox}

\begin{algobox}{Certifier $\C_{n}(S)$ run on input sample $S$:}\label{alg:certifier_majority}
    \begin{enumerate}[itemsep=0pt]
    \item Let $O=\tunif_{\X_n}(\xi(n),\xi(n),S)$.\Comment{Algorithm~\ref{alg:uniformity test} for testing against $\D(n)$}
    \item If $|S| \geq m(n)$ and $O=1$ then return $5\xi(n)$, else return 1.
\end{enumerate}
\end{algobox}

It remains to show that no semi-supervised oracle learner can be relatively smart with respect to the above certifiable error rate for any labeled sample complexity blow-up $\sigma_\ell(m,\eta) = O_{\eta}(m^{2-\beta})$, regardless of the amount of unlabeled data available. In fact, we prove the stronger statement in which the learner is given the marginal distribution $\D(n)$ itself, rendering unlabeled data entirely unnecessary. 

Fix any oracle learner $\A$ with access to the marginal distribution $\D(n)$. We show that with positive probability over a randomly chosen $g_n\in \{0,1\}^n$, $\A$ will have error at least $1/2 - o(1)$ on $f_{n,g_n}$ with respect to $\D(n)$ when the sample size is $ m < M(n) \approx n^{1-\beta}$. 
We prove this by defining, for every $G$, an oracle $\O_{\HH(G)}$ that always favors the hypotheses from the lopsided class. Formally, for any labeled set $T$, if there exists $h\in \bigcup_{n\in \NN} \H(n)$ with $h \in \argmin_{h\in \HH(G)} L(h,T)$ then $\O_{\HH(G)}(T)$ returns a hypothesis $h \in \argmin_{h\in \HH(G)} L(h,T)$ with a fixed tie-breaking rule independent of $G$, otherwise, it will return $f_{i,g_i}$ with the smallest index $i$ among all $f_{i,g_i}\in \argmin_{h\in \HH(G)} L(h,T)$.

For a row of interest $n$, we say that a query $T$ \emph{reveals $g_n$} if the oracle
returns $f_{n,g_n}$. We will prove that with polynomially many queries, with high probability over a randomly chosen $g_n$, the oracle $\O_{\HH(G)}$ does not reveal $f_{n,g_n}$. 
We first show that, conditioned on the observed label sample of size $m \leq M(n)$ with $s\leq m$ distinct instances, a fixed query reveals $g_n$ only with exponentially small probability. Throughout the remainder of the proof, when probabilities are taken over $g_n$, all labelings $(g_j)_{j\neq n}$ are considered fixed. 

\begin{restatable}{lemma}{fixedQueryReveal}\label{lemma:fixed_query_reveal}
For any $n\in\NN$,  $S_0\subseteq \X_n$ with $|S_0|=s<M(n)$, and any arbitrary
labeling $a\in\{0,1\}^{S_0}$ we have  for every finite labeled multiset $T\in(\X\times\{0,1\})^*$ that,
\[
\probs{g_n}{ L(f_{n,g_n},T) < \min_{h\in\H(n)}L(h,T) \,\middle|\, g_n|{S_0}=a }\leq  \exp\left(-\frac{(M(n)-s)^2}{2(n-s)} \right)
\]
\end{restatable}

The proof of the above lemma can be found in Appendix~\ref{app:pf_fixed_query_reveal}. It is based on bounding the number of labelings on $\X_n$ that can strictly outperform the lopsided class on a fixed
query. Suppose that two labelings $g,g'$ that are consistent with the observed labels both have
strictly smaller empirical loss than every hypothesis in $\H(n)$. Assume
$g$ and $g'$ are more than $n-M(n)$ far apart.  We can construct two hypotheses $h$ and $h'$ in
$\H(n)$ as follows: on every point $x\in \X_n$ on which $g$ and $g'$ agree, both
$h(x)$ and $h'(x)$ are equal to $g(x)=g'(x)$, while on their disagreement region, $h$ is constant $0$ and $h'$ is constant $1$.  Since $g$ and $g'$ agree on at most $M(n)$ points,
we have $h,h'\in\H(n)$. Moreover, this construction yields $h(x)+ h'(x) = g(x)+g'(x)$ for every $x\in \X_n$. Considering $g,g'$ and $h,h'$ are constant $1$ on $\X \setminus \X_n$, the sum of the error of $g$ and $g'$ on the queried set equals the sum of the error of $h$ and $h'$ on the same set. Consequently, it is impossible for both $g$ and $g'$ to strictly outperform every member of $\H(n)$.  This
bounds the \emph{diameter} of the set of labelings that can be revealed by
a fixed query. We then get from Lemma~\ref{lemma:cube-diameter} that the size of every subset with diameter at most $n-M(n)$ is exponentially small which concludes that the claimed bound on the probability that a random $g_n$ is revealed by a query.

We now extend the fixed-query bound to polynomially many adaptive oracle calls: if the oracle ever returns $f_{n,g_n}$, there must be a first round $t$ at which this occurs. Consider first the case that $f_{n,g_n}\notin\H(n)$. By construction of $\O_{\HH(G)}$ and its tie-breaking rule, up to round $t$, all oracle responses are determined independent of the unseen labels of $g_n$; consequently, the query at round $t$ is fixed with respect to the randomness in the labeling of unobserved points.  Now for the hypothesis $f_{n,g_n}$ to be returned at round $t$, if $f_{n,g_n}\notin \H(n)$, it must have strictly smaller empirical loss on this query than every member of $\H(n)$. Since, conditional on the observed labels, this query is independent of the remaining labels of $g_n$, Lemma~\ref{lemma:fixed_query_reveal} shows that this occurs with exponentially small probability for each possible $t$. A union bound over polynomially many oracle calls therefore shows that the probability that $f_{n,g_n}\notin\H(n)$ and the oracle nevertheless returns $f_{n,g_n}$ is $o(1)$. Moreover, the probability that $f_{n,g_n}\in\H(n)$ is itself exponentially small. Combining these, we conclude that with probability $1-o(1)$ the oracle never reveals $f_{n,g_n}$, as formalized in the following lemma. The proof of this lemma appears in Appendix~\ref{app:pf_poly_query_reveal}.

\begin{lemma}\label{lemma:poly_query_reveal}
For any $n\in\NN$, any $S_0\sse\X_n$ with $|S_0|=s<M(n)$, and any arbitrary labeling $a\in\{0,1\}^{S_0}$, if $\A$ makes at most $K$ adaptive calls to $\O_{\HH(G)}$, then, conditioned on the test point $x$ and on the internal randomness of $\A$, we have
\[
\probs{g_n}{\text{$\O_{\HH(G)}$ returns $f_{n,g_n}$ in some oracle call}\,\middle|\,g_n|{S_0}=a} \leq (K+1)\exp\left( -\frac{(M(n)-s)^2}{2(n-s)}\right).
\]
\end{lemma}

The final piece of the proof is to translate the bound on the probability that the oracle reveals $f_{n,g_n}$ into a lower bound on the prediction error of $\A$.  For a fixed labeled sample, the label $g_n(x)$ of every point $x\in\X_n$ that does not appear in the sample remains a uniformly random bit.  On the other hand, by construction of the oracle, fixing the sample, test point, and internal randomness of $\A$ determines a unique sequence of oracle responses as long as none of the responses is $f_{n,g_n}$. 
Hence, along this sequence, the learner makes a unique prediction that is independent of the unseen labels of $g_n$, and hence, on a test point not appearing in the sample, it has error exactly $1/2$ over the randomness of $g_n$. The actual interaction with the oracle differs from this fixed interaction only if $f_{n,g_n}$ is revealed.  The bound in Lemma~\ref{lemma:poly_query_reveal} on the probability that $f_{n,g_n}$ is revealed then shows that the expected error error of $\A$, over randomness of $g_n$, is at least $1/2-o(1)$, which proves the existence of $g_n^*$ with the claimed error bound. This is formally captured in the following lemma, the proof of which is relegated to Appendix~\ref{app:pf_oracle_expected_error}.
\begin{lemma}\label{lemma:oracle_expected_error}
For any $n\in\NN$, and $m<M(n)$, if $\A$ makes at most $K$ adaptive calls to $\O_{\HH(G)}$, then
\[
\expects{g_n}{\epsilon_{\A}(\D(n),m)}\geq
\frac12\left(1-\frac{m}{n}\right) -(K+1)\exp\left(  -\frac{(M(n)-m)^2}{2(n-m)}\right).
\]
\end{lemma}

We can now complete the proof of Theorem~\ref{thm:oracle_impossibility}. For any sample size $m \leq M(n)/2$, since $\A$ is an efficient oracle learner, the number $K$ of oracle calls it makes is polynomially bounded in $m \leq n^{1-\beta}$ and we have $K = n^{O(1)}$.
Moreover, since $m\leq M(n)/2$, we have
$ \frac{(M(n)-m)^2}{2(n-m)}\geq \frac{M(n)^2}{8n} \geq \frac18 n^{1-2\beta}$ and hence for sufficiently large $n$, 
$(K+1) \exp\left(-\frac{(M(n)-m)^2}{2(n -m)}  \right) \leq n^{O(1)} \exp\left( -\frac18 n^{1-2\beta} \right)\leq n^{-\beta}
$. It follows that, for all sufficiently large $n$, \[
\expects{g_n}{\epsilon_{\A}(\D(n),m)}\geq
\frac{1}{2} - 2n^{-\beta}.
\]

The above concludes that for all sufficiently large $n$, any sample size $m \leq M(n)/2$ there exists some $G^*$ such that for $\HH(G^*)$ we have that $\epsilon_{\A}(\D(n),m)\geq
\frac{1}{2} - 2n^{-\beta}$ as long as $m\leq M(n)/2 \approx n^{1-\beta}/2$. On the other hand, by Lemma~\ref{lemma:certifiable_A_f_error}, for the same class $\HH(G^\star)$ we have the certifiable error rate
$ \epsilon_n(\D(n),m(n)) = O(n^{-\beta})$. Since $M(n)/2 = \omega\left((m(n))^{2-14\beta}\right)$, we get that $\A$ cannot approximate this certifiable error rate within any multiplicative constant $\alpha$  and any additive constant $\eta<\frac{1}{2}$, both independent of $n$. This concludes that no learner can be $\sigma_\ell(m,\eta) = O_{\eta}(m^{2-14\beta})$ relatively smart. The above construction was for a fixed $\beta$. To obtain a single domain that rules out every substantially subquadratic labeled sample complexity, we take a disjoint union over a sequence of parameters $\beta$ tending to zero. Specifically, for each $\beta\in B=\{1/i: i\in\NN,\ i>8\}$, let $\X_\beta$ be a distinct copy of the domain used in the above construction, and let $\X=\biguplus_{\beta\in B}\X_\beta$. Now fix any substantially subquadratic labeled sample complexity $\sigma_\ell$. We can choose $\beta\in B$ sufficiently small, sot that the preceding argument constructs a choice $G^*$ and a hypothesis class $\HH_\beta(G^*)$ on $\X_\beta$ for which $\A$ fails to be relatively smart with labeled sample complexity $\sigma_\ell$. Extending every hypothesis in $\HH_\beta(G^*)$ to be constant $1$ outside $\X_\beta$ gives a hypothesis class over the common countable domain $\X$, thereby proving the claimed impossibility result.
\hfill\qed}

\section{AI Disclosure}
ChatGPT-5.6 was used to help with proving some of our results. The authors asked GPT 5.6 Pro to prove Theorem~\ref{thm:erm_smart}. The model first came up with a proof that ERM is relatively-$O(\frac{m^2}{\eta}\log(1/\eta))$ smart. Subsequent interactions with the model then resulted in the current proof approach for ERM being relatively $O(m^2/\eta)$-smart. The definitions, questions, and results of Section 4, posing relatively smart semi-supervised learning and establishing the associated positive result, were derived by the human authors without any help from AI.

The question of tractable semi-supervised learning, and whether gains from unlabeled data can be exploited by a simple learner, was also initiated by the authors. The authors formalized tractability through access to the hypothesis class only via calls to an agnostic ERM oracle. With assistance from GPT-5.6, the authors generalized the construction from Theorem~3.1 of \citet{pmlr-v336-dughmi26a} to obtain the class used in Theorem~\ref{thm:oracle_impossibility}, which witnesses the impossibility of tractable learning in this model. GPT-5.6 also assisted in developing parts of the proof. Several of its suggestions contained substantive flaws, which were identified and corrected by the authors.

All the proofs are written by human authors. Any contribution by GPT is rederived, formalized, and written by the authors. The paper was overwhelmingly written by authors. GPT was only used to marginally sharpen some of the writing in order to improve readability. The authors take full responsibility for all the content and proofs in this paper. 

%\bibliography{learning}
%\addbibresource{learning}
\printbibliography
%\bibliography{learning}
%\bibliographystyle{plainnat}
%\bibliographystyle{alpha}
%\bibliographystyle{abbrvnat}               %latex8
%\bibliographystyle{plain}
%\bibliographystyle{abbrv}               %latex8

\appendix

\ifthenelse{\boolean{submission}}{}{\section{Missing Proof from Section~\ref{sec:erm_smart}}} \subsection{Proof of Lemma~\ref{lemma:random-split}}\label{app:pf_erm_smart}
For a labeled set $S$ of $n$ i.i.d.~samples from $\D_{h'}$, denote by $A_S$ the event $\{\exists h \in \H: L(h,S) = 0 ,\, L(h,\D_{h'}) \geq \gamma \}$. For any $h$ and any finite multiset $T\in (\X \times \{0,1\})^*$ denote $N(h,T) = \sum_{(x,y)\in T}\id \{h(x) \neq y\}$. We wish to bound the probability of $A_S$. Now consider two (independent) sets $S$ and $T$ each of $n$ i.i.d.~draws from $\D_{h'}$. Denote by $B_{S,T}$ the event $\{\exists h \in \H: L(h,S) = 0, \, L(h,\D_{h'}) \geq \gamma,\, N(h,T) \geq n \gamma/2\}$. Fix any set $S$ for which there exists some $h_S \in \H$ with $L(h,S)=0$ and $L(h,\D_{h'})\geq \gamma$. We have $\probs{T\sim \D_{h'}^n}{B_{S,T}} \geq \probs{T\sim \D_{h'}^n}{N(h_S,T) \geq n \gamma/2}$. By a Chernoff bound and since $n \geq \frac{8\ln 2}{\gamma}$, we have that $\probs{T\sim \D_{h'}^n}{N(h_S,T) \geq n \gamma/2} \geq 1/2$. Hence, we get that $\probs{S \sim \D_{h'}^n}{A_S} \leq 2 \probed{S\sim \D_{h'}^n}\probs{T \sim \D_{h'}^n}{B_{S,T}}$.

Consider a set $Z=(S,T)$ with $|S|=|T|=n$ of $2n$ i.i.d~sample from $\D_{h'}$. Denote by $E_{Z}$ the event $\{\exists h\in \H_{\dom(Z),\gamma/16}: N(h,S)\leq n\gamma/8, N(h,T)\geq 3n\gamma/8\}$, where $\H_{\dom(Z),\gamma/16}$ is the maximal $\gamma/16$-packing of $\H$ under $\D_{\dom(Z)}$. Note that $\H_{\dom(Z),\gamma/16}$ is also a $\gamma/16$-cover of $\H$ under $\D_{\dom(Z)}$. Therefore, we can easily see that for any $Z=(S,T)$ with $|S|=|T|=n$, the event $B_{S,T}$ implies the event $E_Z$. Hence,  $\probed{S\sim \D_{h'}^n}\probs{T \sim \D_{h'}^n}{B_{S,T}} \leq \probs{Z=(S,T)\sim \D_{h'}^{2n}}{E_Z}$.

Condition on $|\H_{\dom(Z),\gamma/16}|\leq K$. Observe that the process of sampling two sets $S$ and $T$ of $n$ i.i.d~ samples is the same as sampling a set $Z\sim \D_{h'}^{2n}$ of size $2n$ and randomly splitting it into two sets of size $n$. For a random subset $I\sim [2n]$, let $Z_I$ denote the subset of $Z$ at indices $I$. Therefore, by a union bound we get
\[
\begin{aligned}
    &\probs{Z=(S,T)\sim \D_{h'}^{2n}}{E_Z} \leq \expects{Z=(S,T)\sim \D_{h'}^{2n}}{\sum_{h \in \H_{\dom(Z),\gamma/16}}\id\left\{N(h,S) \leq \frac{n\gamma}{8},\, N(h,T) \geq \frac{3n\gamma}{8}\right\}}\\
    & \leq \expected{Z=(S,T)\sim \D_{h'}^{2n}}\expects{I\sim [2n],|I|=n}{\sum_{h \in \H_{\dom(Z),\gamma/16}}\id\left\{N(h,Z_I) \leq \frac{n\gamma}{8},\, N(h,Z_{[2n]\setminus I}) \geq \frac{3n\gamma}{8}\right\}}\\
    & = \expected{Z=(S,T)\sim \D_{h'}^{2n}}\sum_{h \in \H_{\dom(Z),\gamma/16}}\expects{I\sim [2n],|I|=n}{\id\left\{N(h,Z_I) \leq \frac{n\gamma}{8},\, N(h,Z_{[2n]\setminus I}) \geq \frac{3n\gamma}{8}\right\}}.
\end{aligned}
\]

Now consider some $Z$ and some $h\in \H_{\dom(Z),\gamma/16}$ such that $N(h,Z)\geq 3n\gamma/8$. Denote by $X_h$ the number of errors in the first half of the set after a random split. We know $X_h$ has a hypergeometric distribution with mean $\expect{X_h}\geq 3n\gamma/16$. Therefore, we can apply a Chernoff bound to conclude that 
\[
\prob{X_h \leq \frac{n\gamma}{8}} \leq \prob{X_h \leq \frac{2}{3} \frac{3n\gamma}{16}} \leq \prob{X_h \leq (1-\frac{1}{3}) \expect{X_h}} \leq \exp\left(-\frac{\expect{X_h}}{18}\right)  \leq \exp\left(-\frac{n\gamma}{96}\right).
\]

Therefore, conditioned on $|\H_{\dom(Z),\gamma/16}|\leq K$, we have that 
\[
\begin{aligned}
    \probs{Z=(S,T)\sim \D_{h'}^{2n}}{E_Z} \leq K\exp\left(-\frac{n\gamma}{96}\right).
\end{aligned}
\]

This concludes that
\[
\begin{aligned}
    \probs{S \sim \D_{h'}^n}{A_S} \leq 2\probs{Z\sim \D_{h'}^{2n}}{E_Z} \leq 2\probs{Z\sim \D^{2n}}{\packnum_{\gamma/16}(\H,\D_Z)\geq K} + 2K\exp\left(-\frac{n\gamma}{96}\right),\\
\end{aligned}
\]
as desired.\hfill\qed
\ifthenelse{\boolean{submission}}{}{}
\ifthenelse{\boolean{submission}}{\section{Missing Proofs from Appendix~\ref{sec:formal_pf_oracle_impossibility}}\label{app:formal_pf_oracle_impossibility}}{\section{Missing Proofs from Section~\ref{sec:oracle_impossibility}}\label{app:formal_pf_oracle_impossibility}}

We use the modified uniformity tester $\tunif$ in the construction of $\C_n$ to detect distributions that are close to uniform on the row of interest. The tester $\tunif$, formally described in Algorithm~\ref{alg:uniformity test}, was introduced in \citet{pmlr-v336-dughmi26a} by adding a simple wrapper around the standard uniformity tester $\tunifstandard$ (Algorithm~\ref{alg:uniformity_test_standard}), allowing it to detect distributions with support outside the desired row. The following lemma summarizes the guarantee of $\tunif$. The proof this lemma can be found in \citet{pmlr-v336-dughmi26a}.
\begin{lemma}[Lemma~B.4 in \citet{pmlr-v336-dughmi26a}]\label{lemma:uniformity_tester}
    Denote $m_{\mtest}(n,\xi,\delta):= \frac{18\cdot 64\sqrt{n}\ln(2/\delta)}{\xi^2}$. Let $\X$ be an arbitrary countable set. For any set $Y \sse \X$ of size $|Y| = n$, any  $\xi,\delta \in (0,1)$, any distribution $\D$ on $\X$, and any sample size $m\geq m_{\mtest}(n,\xi/2,\delta)$, the tester $\tunif_Y(\xi,\delta,S)$ (Algorithm~\ref{alg:uniformity test}) satisfies the following:
    \begin{itemize}
        \item If $\D = \D_Y$, 
        $\tunif_Y(\xi,\delta,S)$ returns $1$ with probability $\geq 1-\delta$ over $S\sim \D^m$.
        \item If  $\dtv(\D,\D_Y) > \xi$, $\tunif_Y(\xi,\delta,S)$ returns $0$ with probability $\geq 1-\delta$ over $S\sim\D^m$.
    \end{itemize}
\end{lemma}
\begin{algobox}{Standard Uniformity Tester $\tunifstandard_Y(\xi,\delta,S)$ for a set $Y$ with input parameters $\xi,\delta\in(0,1)$ and sample $S$ \citep{GoldreichRon2000Expansion}.}\label{alg:uniformity_test_standard}
\begin{enumerate}[itemsep=0pt]
    \item Let $n=|Y|$, $\ell=18\ln(2/\delta)$, and $m'= |S|/\ell$.
    \item Divide $S$ into $\ell$ consecutive sub-samples $S_i= (x_{i1},\ldots,x_{im'}), 1 \leq i \leq \ell$.
    \item Let $\textsc{tr} = \frac{1+2\xi^2}{n}$.
    \item For $1 \leq i \leq \ell$ do:
    \begin{enumerate}
         \item Let $Z_i = \frac{1}{{m' \choose 2}} |\{(j,k):j<k, x_{ij}=x_{ik}\}|$.
        \item If $Z_i <\textsc{tr}$ then let $\textsc{acc}_i = 1$, else let $\textsc{acc}_i = 0$.
    \end{enumerate}
    \item If $\sum_{i}\textsc{acc}_i \geq \ell/2$ then return $1$, else  return $0$. 
\end{enumerate}
\end{algobox}

\begin{algobox}{Modified uniformity tester $\tunif_Y(\xi,\delta,S)$ for set $Y$ with parameters $(\xi,\delta)\in(0,1)$ and input sample $S$:}\label{alg:uniformity test}
    \begin{enumerate}[itemsep=0pt]
\item If $S \subseteq Y$ then return $\tunifstandard_Y(\xi/2,\delta,S)$; {\Comment{Algorithm~\ref{alg:uniformity_test_standard} for testing against $\D_Y$}}
\item Else return 0.
\end{enumerate}
\end{algobox}

\subsection{Proof of Lemma~\ref{lemma:certifiable_A_f_error}}\label{app:pf_certifiable_A_f_error}

We prove that the learner $\A_{n,G}$ and certifier $\C_n$ together witness the certifiable error rate $\epsilon_n(.,.)$. In other words, we prove that the certifier $\C_n$ is sound for $\A_{n,G}$ and is never greater than the certifiable error rate. Formally, we prove that for all distributions $\D$ and sample size $m$, we have 
\begin{equation}\label{eq:certifiable_requirement}
\epsilon_{\A_{n,G}}(\D,m)\leq \expects{S \sim \D^m}{\C_n(S)}\leq \epsilon_{n}(\D,m).
\end{equation}
This holds trivially when $m < m(n)$, since we have $\epsilon_{\A_{n,G}} (\D,m) \leq  \expects{S \sim \D^m}{\C_n(S)}= \epsilon_{n}(\D,m) = 1$.

Let $h^*\in\HH(G)$ be the labeling function. We will use the following guarantee throughout the proof to bound the error of $\A_{n,G}$: for any distribution $\D$, large $n$ and $m \geq m(n)$, we can conclude from a Chernoff and union bound that
\begin{equation}\label{eq:h1_h0_concentration}
\probs{S \sim \D_{h^*}^{m}}{\exists h \in \{f_{n,g_n},\hone,\hzero\},\, \left|L(h,S) - L(h,\D_{h^{*}})\right| > \xi(n)} \leq 6\exp\left(-m\xi(n)^2\right) \leq \xi(n),
\end{equation}
where the last inequality is due to the fact that for large $m$, we have  $m \geq \frac{\ln(6/\xi(n))}{\xi(n)^2}$.

We consider three cases based on the underlying distribution $\D$ and use the following guarantee for $\tunif$ to bound the expected value of the certifier $\C_n$; see \citet{pmlr-v336-dughmi26a} for a proof. We will assume that $n$ is always larger than the constant in the lemma for the rest of the proof.

\begin{lemma}[Lemma~C.6 of \citet{pmlr-v336-dughmi26a}]\label{lemma:uniformity_tester_Xn}
    There exists an absolute constant $C_1 = C_1(\beta)\in\NN$ such that for any $n\geq C_1(\beta)$, we have for any $m\geq m(n)$,
     \begin{itemize}
        \item If $\D = \D(n)$, 
        $\tunif_{\X_n}(\xi(n),\xi(n),S)$ returns $1$ with probability at least $1-\xi(n)$ over $S \sim \D^m$.
        \item If $\dtv(\D,\D(n)) > \xi(n)$,  
        $\tunif_{\X_n}(\xi(n),\xi(n),S)$ returns $0$ with probability at least $1-\xi(n)$ over $S \sim \D^m$.
    \end{itemize}
\end{lemma}

We now discuss each case for $\D$ separately.

\begin{enumerate}
%-------------------------CASE1----------------------------
%%-------------------------------------------------------
    \item {$\D = \D(n)$.} 
    We first upper bound the error of $\C_n$ and show it is never larger than $\epsilon_n(\D(n),m) = 6\xi(n)$. We know from Lemma~\ref{lemma:uniformity_tester_Xn} that with probability at least $1-\xi(n)$ over $S \sim \D^{m}$, we have $\tunif_{\X_n}(\xi(n),\xi(n),S)$ accepts the underlying distribution as $\D(n)$ and thus $O = 1$. Therefore, we have $\expects{S \sim \D^{m}}{\C_n(T)} \leq \xi(n) + 5\xi(n) = 6\xi(n)$. 
    
    We now upper bound the error of the learner $\A_{n,G}$. Two possibilities can happen: either $h^* \in \H(n,G)$ or $h^* \notin \H(n,G)$. In the first possibility, either $h^* = f_{n,g_n}$ in which case $f_{n,g_n}$ has zero error or $h^* \in \H(n)$ in which case the probability of the minority label of $h^*$ is $\xi(n)$ and either $\hone$ or $\hzero$ will have error at most $\xi(n)$ on $\D$.

    In the second possibility where $h^* \notin \H(n,G)$ we know that $h^*(x) = 1$ for all $x\in \X_n$. Hence, $\D[\{x: h^*(x) = 1\}] =1$ and therefore $\D[\{x: h^*(x) \neq \hone(x)\}] = 0$ and $\hone$ will have zero error.

    Let $\hat{h} \in \arg\min_{h \in \{f_{n,g_n},\hone,\hzero\}}L(h,\D_{h^*})$.  We showed that in both cases for $h^*$, there exists some $h \in \{f_{n,g_n},\hone,\hzero\}$ such that $L(h,\D_{h^*}) \leq \xi(n)$. This implies that $L(\hat{h},\D_{h^{*}})\leq \xi(n)$.  
    
    From Equation~\eqref{eq:h1_h0_concentration} and since $\A_{n,G}(S) \in \arg\min_{h \in \{f_{n,g_n},\hone,\hzero\}}L(h,S)$, we get with probability at least $1-\xi(n)$ over $S \sim \D^m$ that
\[
L(\A_{n,G}(S),\D_{h^{*}}) \leq L(\A_{n,G}(S),S) + \xi(n) \leq L(\hat{h},S) +\xi(n) \leq L(\hat{h},\D_{h^{*}}) + 2\xi(n)\leq 3\xi(n).
\]
Therefore, noting that $\C_n(S) \geq 5\xi(n)$ for all $S$, we get that 
\begin{equation*}
\begin{aligned}
\epsilon_{\A_{n,G}}(\D(n),m) &= \sup_{h^{*} \in \H} \expects{S \sim \D(n)_{h^*}^m}{L(\A_{n,G}(S),\D(n)_{h^{*}})} \\
&\leq 4\xi(n) 
\leq  \expects{T \sim \D^m}{\C_{n}(T)} \leq 6\xi(n) = \epsilon_n(\D(n),m).
\end{aligned}
\end{equation*}

%-------------------------CASE2----------------------------
%%-------------------------------------------------------
\item $\dtv(\D,\D(n)) \leq \xi(n)$.
In this case, we show that the error of the certifier can still bound the error of the learner. We consider the two possibilities that can happen: $h^{*} = f_{n,g_n}$ or $h^{*} \in \HH(G) \setminus \{f_{n,g_n}\}$.

In the first possibility, we have $L(f_{n,g_n},\D_{h^*}) = 0$. In the second possibility, let $b$ denote the majority label of $h^*$ under $\D$ and let $p = \D[\{x\in\X: h^*(x) = b\}]$. Since $h^* \neq f_{n,g_n}$, we either have $h^*\in \H(n)$ or $h^*\notin \H(n)$. Since $\dtv(\D,\D(n)) \leq \xi(n)$, we observe that for $b=1$, if $ h^{*} \in \H(n)$, we have $\D(n)[\{x\in \X_n: h^{*}(x)=1\}] \geq 1-\xi(n)$ and hence $p\geq 1-2\xi(n)$. This is due to the fact that if by the sake of contradiction, we assume for $h^{*}$ we have $\D(n)[\{x\in \X_n: h^{*}(x)=1\}] \leq \xi(n)$, then we would get $p <2 \xi(n)< 1/2$, which is a contradiction to $b=1$ being the majority label. If $ h^{*} \notin \H(n)$, we have $\D(n)[\{x\in \X_n: h^{*}(x)=1\}] = 1$ and $p\geq 1-\xi(n)$. On the other hand, if $b=0$, we have $h^{*} \in \H(n)$. To see this, assume by the sake of contradiction that $h^{*}\notin \H(n)$. Then we get that $\D(n)[\{x\in \X_n: h^{*}(x)=0\}] = 0$ and thus $\D[\{x\in \X_n: h^{*}(x)=0\}] \leq \xi(n)<1/2$ which is a contradiction to $0$ being the majority label. Therefore, $\D(n)[\{x\in \X_n: h^{*}(x)=0\}] \geq 1-\xi(n)$ and $p \geq 1-2\xi(n)$. 
Thus, in the second possibility we always have that $1-2\xi(n) \leq p \leq 1$. Moreover, observe that $L(\hone,\D_{h^*}) = \D[\{x: h^*(x) \neq 1\}]$ and $L(\hzero,\D_{h^*}) = \D[\{x: h^*(x) \neq 0\}]$. Therefore, in the second possibility we have $\min_{h\in\{\hone,\hzero\}}L(h,\D_{h^*}) \leq 1-p \leq 2\xi(n)$.

We proved in any case that $L(\hat{h},\D_{h^*}) \leq 2\xi(n)$ for $\hat{h} \in \argmin_{h\in\{f_{n,g_n},\hone,\hzero\}} L(\hat{h},\D_{h^*})$. Similar to the previous case, from $\A_{n,G}(S) \in \arg\min_{h \in \{f_{n,g_n},\hone,\hzero\}}L(h,S)$ and Equation~\eqref{eq:h1_h0_concentration}, we get with probability at least $1-\xi(n)$ over $S \sim \D^m$ that
$
L(\A_{n,G}(S),\D_{h^{*}}) \leq L(\hat{h},\D_{h^{*}}) + 2\xi(n)\leq 4\xi(n)$. Since $\C_n(S) \geq 5\xi(n)$ for all $S$, we get that 
\begin{equation*}
\begin{aligned}
    \epsilon_{\A_{n,G}}(\D,m) & = \sup_{h^{*} \in \H} \expects{S \sim \D_{h^*}^m}{L(\A_{n,G}(S),\D_{h^{*}})}\\
    & \leq 5\xi(n) \leq  \expects{T \sim \D^m}{\C_{n}(T)} \leq \epsilon_n(\D,m) = 1.
    \end{aligned}
\end{equation*}

%-------------------------CASE3----------------------------
%%-------------------------------------------------------
\item {\bf $\dtv(\D,\D(n)) > \xi(n)$.} 
  In this case, we can invoke Lemma~\ref{lemma:uniformity_tester_Xn} to get that with probability at least $1-\xi(n)$ over $S \sim \D^{m}$, $\tunif_{\X_n}(\xi(n),\xi(n),S)$ (Algorithm~\ref{alg:uniformity test}) rejects and outputs $0$. Therefore, $\C_{n}$ outputs $1$ with probability at least $1-\xi(n)$ and $\expects{S \sim \D^{m}}{\C_{n}(S)} \geq 1-\xi(n)$. 
  Let $b$ denote the majority label of labeling function $h^{*}\in \HH$ under distribution $\D$ and denote  $p : = \D[\{x: h^{*}(x) =b \}] \geq 1/2$. Since $L(\hone,\D_{h^*}) = \D[\{x: h^*(x) \neq 1\}]$ and $L(\hzero,\D_{h^*}) = \D[\{x: h^*(x) \neq 0\}]$, we have that $L(\hat{h},\D_{h^*}) \leq 1-p \leq 1/2$ for $\hat{h} \in \arg\min_{h \in \{f_{n,g_n},\hone,\hzero\}}L(h,\D_{h^*})$. Therefore, similar to the analysis in the previous case, we get with probability at least $1-\xi(n)$ over $S \sim \D^m$ that $ L(\A_{n,G}(S),\D_{h^{*}}) \leq L(\hat{h},\D_{h^{*}}) + 2\xi(n)\leq 1/2 + 2\xi(n)$. Hence, we get that 
  \begin{equation*}
      \begin{aligned}
          \epsilon_{\A_{n,G}}(\D,m) &= \sup_{h^{*} \in \H} \expects{S \sim \D_{h^*}^m}{L(\A_{n,G}(S),\D_{h^{*}})} \\
          &\leq \frac{1}{2}+3\xi(n) \leq  \expects{T \sim \D^m}{\C_n(T)} \leq \epsilon_n(\D,m) = 1,
      \end{aligned}
  \end{equation*}
where we used the fact that $1/2+3\xi(n) \leq 1-\xi(n)$ for large $n$.\hfill\qed
\end{enumerate}

\subsection{Proof of Lemma~\ref{lemma:fixed_query_reveal}}\label{app:pf_fixed_query_reveal}
We will use the following lemma in the proof of Lemma~\ref{lemma:fixed_query_reveal} to bound the size of subsets of boolean cube that have bounded \emph{diameter}.
\begin{lemma}
\label{lemma:cube-diameter}
 Let $\F \subseteq\{0,1\}^d$. For any $f,f'\in \F$, define $f\Delta f' = \{i\in [d]: f(i)\neq f'(i)\}$ and denote $\diam(\F) =\max_{f,f'} |f\Delta f'|$. Assume that $\diam(\F)\leq d-r-1$ for some integer $0\leq r \leq d-1$. Then, we have
\[
  \frac{|\F|}{2^d} \leq \exp\left(-\frac{r^2}{2d}\right).
\]
\end{lemma}

\begin{proof}
We know from 
\citet{kleitman1966combinatorial} that the bound $2k$ on the diameter of $\F$ implies
\[ |\F| \leq  \sum_{j=0}^{k} {d \choose j}.
\]
Let $k=\left\lfloor\frac{d-r}{2}\right\rfloor$. Since $\diam(\F) \leq d-r-1$, we can apply the above bound to conclude that 
\[|\F| \leq \sum_{j=0}^{\lfloor(d-r)/2\rfloor}{d \choose j}.
\]
Let $Z$ denote the random variable that is the sum of $d$ random bits. It is easy to observe that 
\[
\begin{aligned}
    \frac{|\F|}{2^d}
    &\leq \frac{1}{2^d}\sum_{j=0}^{\lfloor(d-r)/2\rfloor}{d \choose j} 
    &=\prob{ Z\leq \left\lfloor\frac{d-r}{2}\right\rfloor} 
    &\leq
    \prob{Z\le\frac{d-r}{2} } 
    &=
    \prob{Z-\frac {d}{2}\le-\frac r2}.
\end{aligned}
\]
Since $\expect{Z} = d/2$ we can apply a Hoeffding's bound to get that
\[
     \frac{|\F|}{2^d}
    \leq \prob{Z-\frac {d}{2}\leq -\frac {r}{2}}\leq \exp\left( -\frac{2(r/2)^2}{d}\right) = \exp\left(-\frac{r^2}{2d}\right).
\]
\end{proof}
We now restate Lemma~\ref{lemma:fixed_query_reveal} and prove it.
\fixedQueryReveal*
\begin{proof}
Since $f_{n,g_n}$ and every $h\in\H(n)$ agree outside $\X_n$, we know for every $h\in \H(n)$ that \[C_T:= |\{(x,y)\in T: x\notin \X_n,\, h(x) \neq y\}| = |\{(x,y)\in T: x\notin \X_n,\, f_{n,g_n}(x) \neq y\}|.\]
For any $x\in\X_n$ and $b\in \{0,1\}$, define
\[
    N_b^T(x) =  \left|\{(x',y)\in T:x'=x,\ y=b\}\right|,
\]
and let $q_T(x):=N_1^T(x)-N_0^T(x)$.
We can verify that for every $h \in \H(n)$,
\[ L(h,T)  =\frac{1}{T}C_T +\frac{1}{T} \sum_{x\in\X_n}N_{1-h(x)}^T(x) =\frac{1}{|T|}\sum_{x\in\X_n}N_1^T(x) - \frac{1}{T}\sum_{x\in\X_n}q_T(x)h(x) + \frac{1}{T}C_T.
\]  Therefore,
\[
    L(f_{n,g_n},T) < \min_{h\in\H(n)}L(h,T)
\]
if and only if
\begin{equation}\label{eq:query_linear_form}
    \sum_{x\in\X_n}q_T(x)g_n(x) >  \max_{h\in\H(n)} \sum_{x\in\X_n}q_T(x)h(x).
\end{equation}

Now let $ V:=\X_n\setminus S_0$, and note that $|V|=n-s$. Conditioned on $g_n|_{S_0}=a$, the restriction $g_n|V$ is uniformly
random over $\{0,1\}^V$. Define
\[
    E_{T,a} = \left\{  z\in\{0,1\}^{V}:  \sum_{x\in S_0}q_T(x)a(x)  +  \sum_{x\in V}q_T(x)z(x)   > \max_{h\in\H(n)} \sum_{x\in\X_n}q_T(x)h(x)
    \right\}.
\]
and observe that \[
    \probs{g_n}{ L(f_{n,g_n},T) < \min_{h\in\H(n)}L(h,T) \,\middle|\, g_n|{S_0}=a } =\frac{| E_{T,a}|}{2^{n-s}}.
\] 

We therefore bound $ \frac{| E_{T,a}|}{2^{n-s}}$.

For every $z,z'\in\{0,1\}^V$ define their disagreement as $z\Delta z' = \{x\in V: z(x) \neq z'(x)\}$. Denote $ \diam( E_{T,a}) = \max_{z,z'\in E_{T,a}} |z\Delta z'|$.

We claim that $\diam( E_{T,a})  <  n-M(n)$. Assume for the sake of contradiction that there exist $z,z'\in E_{T,a}$ such
that $|z\Delta z'|\geq n-M(n)$.
Let $ J=\{x\in V:z(x)=z'(x)\}$ and note that 
\[
    |J| = |V|- |z\Delta z'| =(n-s)- |z\Delta z'| \leq M(n)-s.
\]
Let $c,c'\in\{0,1\}^{\X_n}$ be the labelings obtained by extending $z,z'$ with $a$ on $S_0$. Define two hypotheses $h_0,h_1 \in \{0,1\}^{\X}$ as follows. For any $x\notin \X_n$, $ h_0(x)=h_1(x)=1$.
For any $x \in \X_n$, let
\[
h_0(x)
=
\begin{cases}
a(x), & x\in S_0,\\
z(x), & x\in J,\\
0, & x\in V\setminus J,
\end{cases}
\qquad \text{and} \qquad
h_1(x)
=
\begin{cases}
a(x), & x\in S_0,\\
z(x), & x\in J,\\
1, & x\in V\setminus J.
\end{cases}
\]
We have $|\{x\in \X_n: h_0(x) = 1\}|\leq |S_0|+|J| \leq  s+(M(n)-s) =  M(n)$,
and hence $h_0\in\H^{(1)}(n)$.  Similarly,
$h_1\in\H^{(0)}(n)$. Therefore, $h_0,h_1\in\H(n)$.

Observe that based on the definition of $h_0$ and $h_1$, for any $x\in \X_n$ we have $h_0(x) = h_1(x)$ if and only if $c(x) = c'(x)$. Hence, $h_0(x) + h_1(x) = c(x) + c'(x)$ for all $x\in \X_n$ and
\[
\frac{1}{2}\left(
    \sum_{x\in\X_n}q_T(x)h_0(x)  + \sum_{x\in\X_n}q_T(x)h_1(x)
\right) = \frac{1}{2}\left(\sum_{x\in\X_n}q_T(x)c(x)  + \sum_{x\in\X_n}q_T(x)c'(x)\right).
\]
Since $z,z'\in E_{T,a}$, we know \[\sum_{x\in\X_n}q_T(x)c(x) > \max_{h\in\H(n)} \sum_{x\in\X_n}q_T(x)h(x)\] and \[\sum_{x\in\X_n}q_T(x)c'(x) > \max_{h\in\H(n)} \sum_{x\in\X_n}q_T(x)h(x),\] and hence
\[
\frac{1}{2}\left(\sum_{x\in\X_n}q_T(x)c(x)  + \sum_{x\in\X_n}q_T(x)c'(x)\right) > \max_{h\in\H(n)} \sum_{x\in\X_n}q_T(x)h(x).
\]
On the other hand, $h_0,h_1\in\H(n)$ and we have 
\[
\frac{1}{2}\left(
    \sum_{x\in\X_n}q_T(x)h_0(x)  + \sum_{x\in\X_n}q_T(x)h_1(x)
\right) \leq \max_{h\in\H(n)} \sum_{x\in\X_n}q_T(x)h(x),
\]
which is a contradiction. This proves
$ \diam(E_{T,a})  <  n-M(n)$.

We can now apply Lemma~\ref{lemma:cube-diameter} (with $d = n-s$ and $r = M(n)-s$) to conclude that 
\[
    \frac{|E_{T,a}|}{2^{n-s}}
    \leq \exp\left(-\frac{(M(n)-s)^2}{2(n-s)} \right),
\]
Hence, we have \[
    \probs{g_n}{ L(f_{n,g_n},T) < \min_{h\in\H(n)}L(h,T) \,\middle|\, g_n|{S_0}=a }\leq  \exp\left(-\frac{(M(n)-s)^2}{2(n-s)} \right),
\]
as desired.
\end{proof}

\subsection{Proof of Lemma~\ref{lemma:poly_query_reveal}}\label{app:pf_poly_query_reveal}
Let $E$ denote the event that $\O_{\HH(G)}$ returns $f_{n,g_n}$ in at least one oracle call, and let $B_n:=\{f_{n,g_n}\in\H(n)\}$. Since any member of $\H(n)$ has minority size at most $M(n)$, we can easily prove that 
\[ \probs{g_n}{B_n\mid g_n|{S_0}=a}
\leq 2\exp\left( -\frac{(n-s-2M(n))^2}{2(n-s)} \right).
\]
We can write
\[
\probs{g_n}{E\mid g_n|{S_0}=a} \leq
\probs{g_n}{B_n\mid g_n|{S_0}=a} + \probs{g_n}{E\cap B_n^c\mid g_n|_{S_0}=a}.
\]
We now bound the second term. For each \(t\in[K]\), let \(E_t\) denote the event that \(B_n^c\) occurs and \(f_{n,g_n}\) is returned for the first time at the \(t\)th oracle call. Then $E\cap B_n^c=\bigcup_{t=1}^K E_t$.

Fix $t\in[K]$. Consider the query $T_t$ made by the learner at round $t$ along the interaction in which $f_{n,g_n}$ has not been returned in any of the first $t-1$ rounds. By construction of the oracle, all oracle responses before the first return of $f_{n,g_n}$ are determined without using the unseen labels of $g_n$. Hence, conditioned on the labeled sample, the test point, and the learner's internal randomness, $T_t$ is fixed with respect to the remaining randomness in $g_n|{\X_n\setminus S_0}$.
If $E_t$ occurs, then the actual interaction has not returned $f_{n,g_n}$ during the first $t-1$ rounds, and therefore the actual query at round $t$ is precisely $T_t$. Therefore, for $E_t$ to occur, we must have
\[    L(f_{n,g_n},T_t)  < \min_{h\in\H(n)}L(h,T_t).
\]
Since $T_t$ is fixed with respect to the unseen labels of $g_n$,
Lemma~\ref{lemma:fixed_query_reveal} gives
\[ \probs{g_n}{ E_t\,\middle|\, g_n|_{S_0}=a}
\leq \exp\left(  -\frac{(M(n)-s)^2}{2(n-s)}\right).
\]
Taking a union bound over all $t\in[K]$,
\[
\begin{aligned}
&\probs{g_n}{ \text{$\O_{\HH(G)}$ returns $f_{n,g_n}$ in some oracle call} \,\middle|\, g_n|{S_0}=a} \\
&\leq \probs{g_n}{B_n\mid g_n|{S_0}=a}  + \sum_{t=1}^K \probs{g_n}{ E_t \,\middle|\, g_n|{S_0}=a} \\
&\leq (K+1)\exp\left( -\frac{(M(n)-s)^2}{2(n-s)}\right),
\end{aligned}
\]
where we used the fact that for sufficiently large $n$, $2\exp\left( -\frac{(n-s-2M(n))^2}{2(n-s)} \right) \leq \exp\left( -\frac{(M(n)-s)^2}{2(n-s)}
\right)$. This completes the proof.\hfill \qed

\subsection{Proof of Lemma~\ref{lemma:oracle_expected_error}}\label{app:pf_oracle_expected_error}

Fix the labeled sample $S$, and let $S_0 =\dom(S)\sse\X_n$ be the set of distinct points appearing in $S$.  Denote  $s=|S_0|\leq m$
and let $a:=g_n|{S_0}$ be the observed labeling. Fix a test point $x$ and also the internal randomness of $\A$.  

Conditioned on all above, $g_n|{\X_n\setminus S_0}$ is uniformly distributed over $\{0,1\}^{\X_n\setminus S_0}$.
Consider the prediction $\widetilde y$ of $\A$ determined by following the sequence of oracle interactions for which $f_{n,g_n}$ is never returned. By construction of the oracle, $\widetilde y$ is determined independently of $g_n|{\X_n\setminus S_0}$ and has expected error $1/2$ on any $x\in \X_n\setminus S_0$. Let $E$ denote the set of all $g_n$ for which $\O_{\HH(G)}$ returns $f_{n,g_n}$ in at least one oracle call. Whenever $E$ does not occur, the actual oracle interaction coincides with the above interaction, and hence the prediction of $\A$ is $\widetilde y$. Therefore, for every $g_n$ we have 
\[
\id\left[ \A(\D(n),S)(x)\neq g_n(x)
\right] \geq
\id[\widetilde y\neq g_n(x)]
-
\id[E].
\]
Consequently, we get for any $x\in \X_n \setminus S_0$ that 
\[
    \probs{g_n}{
    \A(\D(n),S)(x)\neq g_n(x)
    \,\middle|\,
    g_n|{S_0}=a} \geq
\frac{1}{2}-\probs{g_n}{E}\]
Since $\D(n)$ is uniform on
$\X_n$, we know $\probs{x\sim\D(n)}{x\notin S_0} =1-s/n$. Thus, using Lemma~\ref{lemma:poly_query_reveal} to bound the probability of $E$, we get
\[
\expected{g_n}\expects{x\sim\D(n)}{
    \id\left[ \A(\D(n),S)(x)\neq g_n(x)
    \right] \,\middle|\, g_n|{S_0}=a}
\geq \frac {1}{2}\left(1-\frac{s}{n}\right) -(K+1)\exp\left( -\frac{(M(n)-s)^2}{2(n-s)}\right).
\]
Since $s\leq m$, we have $1-s/n \geq 1-m/n$.
Moreover, since $s,m <M(n)$ we know
\[
    \frac{(M(n)-s)^2}{n-s}
    \geq
    \frac{(M(n)-m)^2}{n-m}.
\]
Averaging over the labeled sample and the internal randomness of
$\A$ concludes 
\[
\expected{g_n}\expects{x\sim\D(n)}{
    \id\left[ \A(\D(n),S)(x)\neq g_n(x)
    \right] }
\geq \frac {1}{2}\left(1-\frac{m}{n}\right) -(K+1)\exp\left( -\frac{(M(n)-m)^2}{2(n-m)}\right).
\]
For every fixed $g_n$, since $f_{n,g_n}\in\HH(G)$, we know
$\epsilon_{\A}(\D(n),m)$ is at least the expected error of $\A$ on the
target $f_{n,g_n}$.  Averaging this inequality over $g_n$ proves the claim.\hfill\qed

\end{document}